\documentclass{article}

\PassOptionsToPackage{numbers,sort&compress}{natbib}

\usepackage{microtype}
\usepackage{graphicx}
\usepackage{booktabs} 

\usepackage[colorlinks,citecolor=blue,urlcolor=black,linkcolor=blue,pdfborder={0 0 0}]{hyperref}

\PassOptionsToPackage{numbers}{natbib}
\usepackage[preprint]{neurips_2020}

\usepackage{array}
\usepackage{xr}
\usepackage{mathtools}
\usepackage{amsmath,amssymb,amsthm}
\usepackage{url}
\usepackage{comment}
\usepackage{graphicx, floatrow}
\usepackage{enumitem}
\usepackage{wrapfig}
\usepackage{lipsum}

\usepackage{framed}
\usepackage{subcaption}
\usepackage{booktabs}

\usepackage[capitalize]{cleveref}  
\usepackage{datetime}

\crefname{nlem}{Lemma}{Lemmas}
\crefname{nprop}{Proposition}{Propositions}
\crefname{ncor}{Corollary}{Corollaries}
\crefname{nthm}{Theorem}{Theorems}
\crefname{exa}{Example}{Examples}
\crefname{assumption}{Assumption}{Assumptions}
\crefname{equation}{}{}

\RequirePackage[usenames,dvipsnames]{color}
\usepackage{mathrsfs}

\usepackage{autonum} 

\graphicspath{ {images/} }

\allowdisplaybreaks[3]

\usepackage{color,mathtools}
\usepackage{pifont}
\newcommand{\eqdef}{\mathbin{\stackrel{\rm def}{=}}}

\usepackage{algorithmic}
\usepackage[ruled]{algorithm2e}
\usepackage{multicol}
\usepackage{autonum}
\usepackage{enumitem}
\setlist[itemize]{leftmargin=3mm}
\setlist[enumerate]{leftmargin=3mm}

\DeclareMathOperator*{\argmin}{arg\,min}
\DeclareMathOperator*{\argmax}{arg\,max}

\def\R{{\mathbb{R}}}

\def\bbE{{\mathbb{E}}}

  \usepackage{nth}
  \usepackage{intcalc}

\newtheorem*{rep@theorem}{\rep@title}
\newcommand{\newreptheorem}[2]{%
\newenvironment{rep#1}[1]{%
 \def\rep@title{#2 \ref{##1}}%
 \begin{rep@theorem}}%
 {\end{rep@theorem}}}
\newreptheorem{lemma}{Lemma'}
\newreptheorem{definition}{Definition'}
\newreptheorem{proposition}{Proposition'}
\newreptheorem{theorem}{Theorem'}

\newtheorem{theorem}{Theorem}

\newtheorem{assumption}[theorem]{Assumption}

\newtheorem{definition}[theorem]{Definition}

\newtheorem{lemma}[theorem]{Lemma}

\newtheorem{remark}{Remark}
\renewcommand{\text}[1]{{\mathrm{#1}}}
\renewcommand{\text}[1]{{\textnormal{#1}}}

\newenvironment{talign*}
 {\csname align*\endcsname}
 {\endalign}
\newenvironment{talign}
 {\csname align\endcsname}
 {\endalign}

\def\balignst#1\ealignst{\begin{talign*}#1\end{talign*}}
\def\balignt#1\ealignt{\begin{talign}#1\end{talign}}

\allowdisplaybreaks
\allowdisplaybreaks[4]

\usepackage{geometry}                

\include{defs}

\title{Model-specific Data Subsampling with Influence Functions}

\author{%
  Anant Raj\\
  MPI for Intelligent Systems\\
 T\"ubingen, Germany\\
  \texttt{anant.raj@tuebingen.mpge.de} \\  
   \And
 Cameron Musco \\
  UMass Amherst\\
 Massachusetts, USA\\
  \texttt{cmusco@cs.umass.edu} \\
   \AND
Lester Mackey\\
  Microsoft Research\\
 New England, USA\\
  \texttt{lmackey@microsoft.com} \\
   \And
  Nicolo Fusi\\
  Microsoft Research\\
 New England, USA\\
  \texttt{fusi@microsoft.com} \\
}

\begin{document}

\maketitle

\begin{abstract}
Model selection requires repeatedly evaluating models on a given dataset and measuring their relative performances. In modern applications of machine learning, the models being considered are increasingly more expensive to evaluate and the datasets of interest are increasing in size. As a result, the process of model selection is time-consuming and computationally inefficient. In this work, we develop a model-specific data subsampling strategy that improves over random sampling whenever training points have varying influence. Specifically, we leverage influence functions to guide our selection strategy, proving theoretically and demonstrating empirically that our approach quickly selects high-quality models.

\end{abstract}

\section{Introduction}
\label{sec:intro}
Model selection and parameter tuning require repeatedly training different models with different parameter settings and evaluating their performance on a holdout set. 
Since repeated training on a large dataset is often prohibitively expensive, it is common to perform this search on a randomly chosen subset of the data and only refit the final model to the entire training set once the best model has been identified. 
Recently, a number of automated algorithms have been developed to facilitate this task. Two examples are Hyperband \citep{li2017hyperband} and BOHB \citep{falkner2018bohb}. They are both based on the same broad motivation: when exploring hyperparameter settings and selecting models sequentially, choices made early on in the process tend to be suboptimal because the optimization space is largely unexplored and model uncertainty is high. It is therefore sensible not to spent large amounts of the total computational budget fitting candidate models on the entire training set very early on; rather the budget is better spent fully evaluating only promising models. In many of such proposed methods, a set of candidate models is sampled at random and evaluated given a fixed computational budget. The computational budget can be controlled by only training models for a certain number of epochs or, more relevant for this work, only training them on a subset of the training set. A fraction of the worst models (in terms of validation accuracy) evaluated using a given subset are then discarded, and the remaining ones are evaluated on a larger subset.  Unfortunately, the quality of the judgments depends on the fidelity of the subsample, which may be low if simple random subsampling is used.

\paragraph{Our Contributions} \label{subsec:contributionss}

In this work, we propose a data subsampling strategy that, for a given model or hyperparameter setting, leverages influence functions \citep{hampel2011robust} to improve over the standard strategy of simple random sampling. Given a subset  of datapoints $z_1,\ldots, z_M$, the influence function
of a datapoint $z$ measures how adding $z$ to the subset with an infinitesimally small weight will affect a model trained on that subset. When $M$ is large enough, the influence permits us to faithfully and inexpensively estimate the performance of a model trained on $z_1,\ldots, z_M$ and $z$ without actually retraining. We use this approximation to greedily add points to our subset with the goal of minimizing a measure of the trained model's error -- such as its loss over a validation set. 
We then use these influential points to (i) rapidly estimate the test performance of the model trained on a much larger dataset, (ii) rapidly estimate the test performance of other models trained on a much larger dataset, and (iii) rapidly judge which of a collection of models will lead to the best test performance.

Our primary contributions are:

\begin{itemize}
\item We prove theoretically and show on three real classification and regression tasks that our greedy subsampling strategy more rapidly minimizes a smooth target functional of a trained model than random sampling whenever training points have sufficiently variable influence. 
This holds for selecting points one at a time or in minibatches of size $m$.
We target validation loss in our experiments, but the method is general and immediately applicable to other functionals. 
\item We show on real classification and regression tasks that our subsamples selected for one model also lead to more faithful approximations of other models' performances; this is especially valuable for models like boosted decision trees that do not admit readily computable influence functions. 
\item In an application to parameter selection problems, we demonstrate empirically that replacing random sampling with influence-based subset selection leads to the identification of more accurate configurations in less time.
\end{itemize}

\subsection{Related Work}
\label{subsec:rel_work}
Influence functions have long been studied in the field of robust statistics. For example, \citet{hampel2011robust} use influence functions to study the impact of data contamination on statistical procedures. In machine learning, influence functions have been used by \cite{christmann2007consistency} to reason about the stability of kernel regression, by \cite{koh2017understanding} to study perturbations of neural networks, and by \cite{giordano2018return} to obtain inexpensive approximations of resampling procedures like cross-validation and the bootstrap. \citet{ting2018optimal} consider single-shot sampling, as compared to our iterative approach, using influence functions, motivating their strategy  as minimizing the variance of an asymptotically linear estimator. Their algorithm requires computing influence functions with the true optimal parameter, which is unknown and so must be approximated. If this approximation is poor (e.g., if obtained from a small random sample) the influence approximations are also poor. In contrast, we update the influence estimate each time we select a new point, obtaining a more informative sample.

Our aim is to use influence functions for principled subset selection. There is a vast literature on this topic. A common goal is to select a (weighted) subset of data points such that the value of some loss function on the full dataset is well approximated by the loss on the subset. In this way, the loss can be approximately minimized by just minimizing over the subset. In the literature, this is known as a \emph{coreset} guarantee. See \citet{bachem2017practical} for a survey, applications of the technique to clustering in \cite{har2004coresets, feldman2007ptas, lucic2017training}, and applications of the technique to logistic regression in  \cite{huggins2016coresets, campbell2017automated, munteanu2018coresets}.
Interestingly, in the cases of ordinary least squares regression, low-rank approximation, and $k$-means clustering, coresets can be obtained by subsampling data points using their \emph{statistical leverage scores} \cite{drineas2006sampling,mahoney2011randomized,clarkson2015sketching,cohen2015uniform,cohen2017input,munteanu2018coresets}. In many cases, these leverage scores are equivalent to or closely related to the influence functions of an appropriate loss function.

Our approach differs from work on coreset construction in a few ways. A coreset approximation guarantee ensures that the model  trained on the subsample has near-optimal loss over the full dataset. Our method will select points to try to minimize some function of the model trained on this subset. While this function may  be the loss over the full dataset, we allow more general choices: e.g., the function could be the loss over a validation set or some other metric capturing how well performance of a model trained on the subset predicts overall performance. Additionally, while coresets have been developed on a case-by-case basis for specific learning problems, our approach has the advantage of being immediately applicable to any sufficiently smooth $m$-estimation task.



\section{Background and Notation}

\subsection{The Subset Selection Problem}
Our goal is to develop a strategy for subsampling data points in order to minimize some cost function over the model trained on this subsample. For example, we may want to minimize loss on a validation set or the expected loss over the data generating distribution (the true risk). 
More formally, consider a set of training points $Z$, a parameter space $\Theta$, and a loss function $\ell: Z \times \Theta \rightarrow \R^+$.
For each distribution $P = \sum_{z \in Z} p(z) \cdot \delta_{z}$ over $Z$ with probability mass function $p$, we define the empirical risk minimizer
\balignt
\hat{\theta}(P) = \argmin_{\theta \in \Theta} \sum_{z\in Z} p(z) \cdot \ell(z,\theta). \label{prob:subsamp}
\ealignt
For a given objective function $R: \Theta \rightarrow \R^+$, our goal is to select a subsample of $s$ points $z_1,\ldots,z_s$ from $Z$ to minimize $R(\hat{\theta}(P^s))$ where $P^s = \frac{1}{s} \sum_{i=1}^s \delta_{z_i}$ is the empirical distribution over $z_1,\ldots,z_s$.
Since solving this problem optimally is generally computationally intractable, we will content ourselves with outperforming the standard subsampling strategy of selecting $s$ points uniformly at random from $Z$.
To this end, we will develop a greedy subset selection strategy based on influence functions and prove 
that this strategy outperforms simple random sampling when datapoint influences are sufficiently variable.

\subsection{Influence Functions}\label{subsec:Influence _Function}
Influence functions approximate the effect that a datapoint $z$ has on an estimator $\hat \theta$ or an objective value $R(\hat \theta)$ when the datapoint $z$ is added to an existing dataset. Formally, the influence function is defined as a Gateaux differential, which generalizes the idea of a directional derivative:


\begin{definition}[{Bouligand Influence Function}] \label{def:influence_fun}
Let $P$ be a distribution and $T$ be an operator $T:P\rightarrow T(P)$ then the
Bouligand influence function (BIF) of $T$ at
$P$ in the direction of a distribution $Q \neq P$ is
defined as: \small$
BIF(Q;T,P) = \lim_{\epsilon \rightarrow 0} \frac{T\left( (1-\epsilon) P + \epsilon Q \right) - T(P)}{\epsilon}$.\normalsize
\end{definition}
The BIF measures the impact of an infinitesimally small perturbation of $P$ in the direction of $Q$ on $T(P)$. Letting  \small$P_{\epsilon,Q} = (1-\epsilon) P + \epsilon Q$\normalsize, it is easy to see that \small$BIF(Q;T,P) = \frac{\partial}{\partial \epsilon} T(P_{\epsilon,Q})\Big|_{\epsilon=0}$ \normalsize.

We typically take $P$ to be the empirical distribution over $M$ samples $\{z_i\}_{i=1}^M$, which we write as $P^M = \frac{1}{M} \sum_{i =1}^M \delta_{z_i}$ and $Q = \delta_{z_{M+1}}$ to be a point mass at an additional sample point $z_{M+1}$. $T$ will be the operator $\hat{\theta}$ returning the empirical risk minimizer \cref{prob:subsamp} over the given distribution. 

For large enough $M$, the effect on $\hat{\theta}$ of adding $z_{M+1}$ to the sample $\{z_i\}_{i=1}^M$ can be approximated using the influence function of $\hat{\theta}$ at the empirical distribution $P^M$ in the direction of the point mass $Q$. This approximation corresponds to a first order Taylor series approximation. 
\subsection{Notation} \label{subsec:notations}
A subset of $M$ data points $\{z_i\}_{i=1}^M$ are being expanded. An additional point selected using proposed greedy approach and using uniform random sampling is denoted by $z_{M+1}^g$  and $z_{M+1}^r$  respectively.  Let $P^M$ be the empirical distribution on $\{z_i\}_{i=1}^M$. 
 \small $ \hat{\theta}_M \eqdef \hat{\theta}(P^M)$ \normalsize denotes the initial model trained on \small $\{z_i\}_{i=1}^M$\normalsize and  \small $\hat{\theta}_{M+1^t} $\normalsize denotes the model trained on \small$P^M_{\epsilon_M,Q_M^t}$\normalsize. We extend this notation, e.g., letting \small$\hat{\theta}_{M+k_1^g+k_2^r+k_3^g}$\normalsize ~ denote the model which is trained by first adding $k_1$ points greedily in sequence, starting at the model \small$\hat{\theta}_{M}$\normalsize, then adding $k_2$ points randomly at  \small$\hat{\theta}_{M+k_1^g}$\normalsize ~ and finally adding $k_3$ points greedily starting at \small $\hat{\theta}_{M+k_1^g+k_2^r}$. $\hat{\theta}_{M_m+k_1^g+k_2^r+k_3^g}$\normalsize ~ can be defined similar to \small $\hat{\theta}_{M+k_1^g+k_2^r}$ \normalsize  where single point is replaced with the block of $m$ points.



\section{First Order Model Approximation with Influence Functions} \label{sec:first_order}
We now show how influence functions can be used to approximate how the addition of a new point into a subsample will affect the trained model $\hat \theta$. We begin by stating the necessary assumptions.
\subsection{Assumptions} \label{subsec:assumptions}
Our assumptions are similar to the ones made in \cite{giordano2018return}. We write the Hessian of the empirical risk minimization problem over the empirical distribution $P^M$ at parameter $\theta$ as 
\balignt
H_{\theta} = \frac{1}{M} \sum_{i =1}^M \nabla^2_{\theta}~ \ell({z}_j,\theta).
\ealignt

\begin{assumption}[Smoothness]\label{as:smooth}
For all  $\theta \in \Theta, z \in Z$, the loss function  $\ell(z,\theta)$ is  twice continuously differentiable in $\theta$. 
\end{assumption}

\begin{assumption}[Non-degeneracy]\label{as:non-degenerate}
The operator norm of  inverse Hessian of the ERM loss function is bounded above. I.e., for all $\theta \in \Theta$ and $\{z_i\}_{i=1}^M \subseteq Z$, $\| H_{\theta}^{-1}\|_{op}\leq C_{op} < \infty$.
\end{assumption}

\begin{assumption}[Bounded gradients]\label{as:bounded}
The norm  of the gradient of the loss function is bounded. I.e., for all $\theta \in \Theta$ and $z \in Z$,  $\| \nabla_{\theta}\ell(z,\theta) \| \leq G < \infty$.
\end{assumption}

\begin{assumption}[Local smoothness] \label{as:local_smooth}
Finally we assume that the Hessian in Lipschitz in the parameter $\theta$  I.e., for all $\theta, \theta' \in \Theta$ and $\{z_i\}_{i=1}^M \subseteq Z$, $\|  H_{\theta} -  H_{\theta'}\|_{op}\leq L \| \theta - \theta' \|$ for some $L\geq0$. 
\end{assumption}

\subsection{Model Approximation} \label{subsec:model_approx}
Intuitively, our assumptions allow us to approximate  a well behaved operator $T$:
\small
\balignt
T(P_{\epsilon,Q}) 
	= T(P) + \epsilon ~BIF(Q;T,P) + \mathcal{O}(\epsilon^2) 
 	= T(P) + \epsilon \frac{\partial}{\partial \epsilon} T(P_{\epsilon,Q})\Big|_{\epsilon=0}   +  \mathcal{O}(\epsilon^2). \label{eq:first_order_approx_T}
\ealignt
\normalsize
As discussed, we let $T$ be the empirical risk minimizer operator $\hat \theta$ \cref{prob:subsamp}. 
After adding another batch of $m$ points $\{{z}_{M+1}^t, \cdots {z}_{M+m}^t \}$ (where $t \in \{g,r\}$ denotes if the points are added  randomly or greedily), we want to approximate:
\small
\balignt
\hat{\theta}_{M_m+1^t} =\argmin_{\theta \in \Omega_\theta} \frac{1}{M+m} \left( \sum_{j =1}^{M} \ell({z}_j,\theta) + \sum_{j=1}^m \ell({z}_{M+j}^t,\theta) \right).
\ealignt
\normalsize
\citet{koh2017understanding} give the following result on the first order approximation of the $\hat{\theta}_{M_m+1^t}$ for $m = 1$ (which we denote as $\hat{\theta}_{M+1^t}$):

\begin{lemma}[{\cite{koh2017understanding}}] \label{lem:1_point_addition_foa}
Let $\tilde{\theta}_{M+1^t} = \hat{\theta}_M + \frac{1}{M+1} BIF(Q_M^t;\hat{\theta},P^M)$ be the first order approximation of the estimator $\hat{\theta}_{M+1^t} $ as in \eqref{eq:first_order_approx_T}. Under Assumptions~\ref{as:smooth}, \ref{as:non-degenerate}, \ref{as:bounded},  and \ref{as:local_smooth}, $BIF(Q_M^t;\hat{\theta},P^M)$ is given as:
\balignt
BIF(Q_M^t;\hat{\theta},P^M) =  - {H_{\hat{\theta}_M} ^{-1}} \nabla_{\theta} \ell(z_{M+1}^t,\hat{\theta}_M) 
\ealignt
and $\| \tilde{\theta}_{M+1^t} -\hat{\theta}_{M+1^t} \| \leq \frac{B_1}{(M+1)^2}$ for some constant $B_1$.
\end{lemma}
We extend Lemma~\ref{lem:1_point_addition_foa} to general $m$ and provide the proofs of both results in Appendix~\ref{ap:foa}.
\begin{lemma} \label{lem:m_point_addition_foa}
Let $\tilde{\theta}_{M_m+1^t} = \hat{\theta}_M + \frac{m}{M+m} BIF(Q_{M_m}^t;\hat{\theta},P^M)$ be the first order approximation of the estimator $\hat{\theta}_{M_m+1^t} $ as in \eqref{eq:first_order_approx_T}. Under Assumptions~\ref{as:smooth}, \ref{as:non-degenerate}, \ref{as:bounded},  and \ref{as:local_smooth}, $BIF(Q_{M_m}^t;\hat{\theta},P^M)$ is given as:
\balignt
BIF(Q_{M_m}^t;\hat{\theta},P^M) &= - \sum_{j=1}^m H_{\hat{\theta}_M}^{-1} ~\nabla_{\theta} \ell({z}_{M+j}^t,\hat{\theta}_M)
\ealignt
and $\| \tilde{\theta}_{M_m+1^t}  - \hat{\theta}_{M_m+1^t}  \| \leq \frac{B_m m^2}{(M+m)^2}$ for constant $B_m$.
\end{lemma}
We see from Lemmas \ref{lem:1_point_addition_foa} and \ref{lem:m_point_addition_foa} that as long as $M$ is large enough, the first order approximation computed using the influence functions will be an accurate approximation for the true updated model. 

\subsection{Approximating Loss Function}\label{sec:lossApprox}
Via a simple application of the chain rule, we can see that, under the same set of assumptions, the influence function can also be used to directly approximate the loss function on any point $z$ after adding a new point $z_{M+1}^{t}$ to the subsample. Specifically we have:
\small
\balignt
\ell(z,\hat{\theta}_{M+1^t} ) - \ell(z,\hat{\theta}_{M} ) 
=& - \frac{1}{M+1} \nabla_{\theta} \ell(z,\hat{\theta}_{M})^\top {H_{\hat{\theta}_M} ^{-1}} \nabla_{\theta} \ell(z_{M+1}^t,\hat{\theta}_M)  +\mathcal{O}\left(\frac{1}{(M+1)^2}\right)
\ealignt
\normalsize
In other words, we have for come constant $B_l$:
\small
\balignt
\large |&\ell(z,\hat{\theta}_{M+1^t} ) - \ell(z,\hat{\theta}_{M} ) + \frac{1}{M+1} \nabla_{\theta} \ell(z,\hat{\theta}_{M})^\top {H_{\hat{\theta}_M} ^{-1}} \large | \le  \frac{B_l}{(M+1)^2} .
\ealignt
\normalsize
Similarly, we can approximate the loss on $z$ when we add $m$ new points to the sample. We have:
\small
\balignt
\ell(z,\hat{\theta}_{M_m+1^t} ) - \ell(z,\hat{\theta}_{M} ) 
 &= - \frac{m}{M+m} \nabla_{\theta} \ell(z,\hat{\theta}_{M})^\top \left( \sum_{j=1}^m H_{\hat{\theta}_M}^{-1} \nabla_{\theta} \ell(z^t_{M+j}, \hat{\theta}_M) \right ) + \mathcal{O}\left(\frac{m^2}{(M+m)^2}\right).
 \ealignt 
 \normalsize
 Thus, for some constant $C_l$:
 \small
 \balignt\label{eq:loss_m_points}
 \big |& \ell(z,\hat{\theta}_{M_m+1^t} ) - \ell(z,\hat{\theta}_{M} )+  \frac{m}{M+m}~\sum_{j=1}^m \underbrace{\left( \nabla_{\theta} \ell(z,\hat{\theta}_{M})^\top H_{\hat{\theta}_M}^{-1} \nabla_{\theta} \ell(z^t_{M+j}, \hat{\theta}_M) \right )}_{\text{Sum of Invidual Influence}} \big | \le \frac{C_l m^2}{(M+m)^2}.
\ealignt
\normalsize
%

We can see from the above that for large $m/M$, the first order approximation to the loss after updating the model by adding one or $m$ new data points to the training set ($\ell(z,\hat{\theta}_{M+1^t})$ and $\ell(z,\hat{\theta}_{M_m+1^t})$ respectively), is a good approximation to the true loss. This will be the key idea behind the greedy algorithm, presenting in Section \ref{sec:greedy_selection} which will chose new points to minimize this approximate loss.
\small
\begin{remark}\label{rem:2}
From equation~\eqref{eq:loss_m_points} and from Lemma~\ref{lem:m_point_addition_foa}, we observe that the effect of adding $m$ points is just a sum over the points' individual effects. We will use this property later to perform a local greedy selection of $m$ points at a time. 
\end{remark}
 \normalsize

\section{Local Greedy Subset Selection} \label{sec:greedy_selection}
\vspace{-4mm}
 \begin{wrapfigure}{L}{0.5\textwidth}
\vspace{-4mm}
 \begin{algorithm}[H]
\begin{algorithmic}
\small
\STATE  \textbf{given: }$\hat{\theta}_M$, dataset $Z$, subsample $\bar{Z}$, budget
  \STATE \textbf{for} {$k=1\dots \text{ budget}$}
  \STATE \quad  $z_{M_m+k} \gets \argmin_{Z_m \in Z \text{\textbackslash} \bar{Z}}   R\left(\argmin_{\theta} L(\theta,Z_m)\right)$ \label{alg_line:exact_greed_erm}
  \STATE \quad $\bar{Z} \gets \bar{Z} \cup z_{M_m+k}$
  \STATE \textbf{return} $\bar{Z}$ 
\end{algorithmic}
\caption{Exact Local Greedy Algorithm}
 \label{algo:exact_greedy_m}
  \normalsize
\end{algorithm}
\vspace{-6mm}
  \end{wrapfigure}
We now discuss how to use the first order approximation from Section \ref{sec:first_order} to greedily select data points to minimize an objective $R(\cdot)$  over the empirical risk minimizer \cref{prob:subsamp}. 
 
  Our greedy approach starts with a subset of points $\{z_i\}_{i=1}^M$, which for convenience we denote as $\bar Z$, along with the empirical risk minimizer over this subset, $\hat{\theta}_M$. Intuitively, we hope to add points iteratively, $m$ at a time to $\bar{Z}$. Ideally, in each iteration, we would choose the subset of $m$ points from $Z$ that, when added, minimize $R\left(\argmin_{\theta} L(\theta, Z_m)\right)$ where \small $L(\theta,Z_m) =  \sum_{z_i  \in \bar{Z} } \ell(z_i,\theta)   +\sum_{\tilde{z}_{j_p} \in Z_m}  \ell(\tilde{z}_{j_p},\theta)$ \normalsize.
 However, this approach (shown in Algorithm~\ref{algo:exact_greedy_m}) is computationally infeasible. In particular, even evaluating the quality  of a subset requires recomputing the empirical risk minimization and the computing $R(\hat \theta)$. On top of this, optimizing over subsets is a combinatorial problem and likely to be difficult. Let $Z_m$ denote the set of $m$ samples from the set $Z\text{\textbackslash} \bar{Z}$.  To deal with this issue, we propose approximating Algorithm~\ref{algo:exact_greedy_m} using the  first order approximation discussed in Section \ref{sec:first_order}. We give an example instantiation of this approach in the case that $R(\hat \theta)$ is the loss over a validation set in Algorithm~\ref{algo:approximate_greedy_m_1_rand} (with $\epsilon = 0$).

Using the approximation results of Section \ref{sec:lossApprox}  in Algorithm~\ref{algo:exact_greedy_m}, we solve the minimization over the first order  approximation,
 in Algorithm~\ref{algo:approximate_greedy_m_1_rand}. In particular, we compute the approximate change in loss when $z_i$ is added to the training set, $\text{\textit{Residual}}(z_i) = \bbE_{ Z'}\Big[\nabla_{\theta} \ell(z',\hat{\theta}) H_{\hat{\theta}}^{-1}  \nabla_{\theta} \ell({z_i},\hat{\theta}) \Big]$ for all $z_i \in Z \setminus \bar Z$. Afterwards,we greedily add the $m$ points to $Z$ that have the largest approximate decreases to the validation loss. Due to the linear nature of the influence function, (see Remark  \ref{rem:2}), this is equivalent to choosing the set of $m$ points that minimize the first order approximation to the validation loss when adding all $m$ points at once.

We discuss the theoretical properties of Algorithm~\ref{algo:approximate_greedy_m_1_rand}  in Section \ref{sec:analysis}, taking into account the error introduced by the first order approximation used.


%
 \vspace{-2mm}
\begin{wrapfigure}{L}{0.6\textwidth}
\vspace{-6mm}
 \begin{algorithm}[H]
\begin{algorithmic}
\small
  \STATE  \textbf{given: } data set $Z$, subsample $\bar{Z}$, validation set $Z'$, iter, $\epsilon$
  \STATE \textbf{for} {$k=1\dots \text{ iter}$}
    \STATE \quad $\hat{\theta} = \argmin_{\theta} \frac{1}{|\bar{Z}|}\sum_{z_i \in \bar{Z} } \ell(z_{i},\theta)$\label{alg_lin:app_g_m_1_lin9}
   \STATE \quad \textbf{with probability $\epsilon$ } : Add $m$ random samples to $\bar{Z}$
\STATE \quad \textbf{with probability $1 - \epsilon$ } :
  \STATE \quad  \quad  \textbf{for} $p=1\dots m$
   \STATE \quad  \quad \quad  $\tilde{z} = \argmax_{z_i \in Z \text{\textbackslash} \bar{Z}}  \bbE_{ Z'}\Big[\nabla_{\theta} \ell(z',\hat{\theta}) H_{\hat{\theta}}^{-1}  \nabla_{\theta} \ell({z_i},\hat{\theta})\Big]$ \label{alg_lin:app_g_m_1_lin6}
  \STATE \quad  \quad \quad $\bar{Z} = \bar{Z} \cup \tilde{z}$
  \STATE \textbf{return} $\bar{Z}$ 
\end{algorithmic}
\caption{Approximate Local $\epsilon$-Greedy Algorithm}
 \normalsize
 \label{algo:approximate_greedy_m_1_rand}
\end{algorithm}
\vspace{-6mm}
  \end{wrapfigure}
\subsection{Practical Algorithm for Risk Minimization: An $\epsilon$-Greedy Approach} \label{sebsec:practical_erm}

Algorithm \ref{algo:approximate_greedy_m_1_rand} with $\epsilon = 0$ invokes our approach when the function $R(\hat \theta)$ is the loss over some validation set. However, it is possible to use influence functions to implement an approximate greedy strategy for more general $R$, as long as we can compute a first order approximation to the change in $R$ when a new data point is added, using e.g., Lemmas \ref{lem:1_point_addition_foa} and \ref{lem:m_point_addition_foa} combined with the chain rule, as we did in Section \ref{sec:lossApprox}.
One function $R$ that we may be particularly interested in minimizing is the population risk. However, when we consider the population risk: we never have access to the population and so cannot compute $R(\hat \theta)$ or its derivative. We can approximate the population using a validation set, however, this approach might quickly overfits on the validation data in case of small size dataset.  To combat this issue, we propose to use $\epsilon$-greedy approach. The idea is to disturb the overfitting by adding random datapoints to avoid  the error on validation set  to saturate. However, for datasets of large size, choosing $\epsilon=0$ works just fine as we would see in the Section~\ref{sec:expts}.
 In principle, it is pretty easy to choose points on the boundary which overfits on the validation set. The addition of randomly selection points force the classifier to generalize on the other samples apart  from the validation set. We provide a formal $\epsilon$\textit{-greedy} algorithm~\ref{algo:approximate_greedy_m_1_rand}.

\paragraph{About Scalability (Hessian):} The main computational bottleneck of this approach comes from the Hessia inversion $H_{\theta}^{-1}$ which might be cubic in worst case. However, efficient computation of Hessian-vector  product is a very well researched field in second order optimization. Conjugate gradient \citep{martens2010deep}, sketching \citep{pilanci2016iterative} or stochastic estimation \cite{agarwal2017second} can be used to efficiently compute $H_{\theta}^{-1}\nabla_{\theta} \ell(z,\theta)$.

\subsection{From One Model to Another}  \label{subsec:model_transfer}
Lemmas~\ref{lem:1_point_addition_foa} and \ref{lem:m_point_addition_foa} suggest that our update method resembles with a newton update. If we assume for a moment that we have access to the full population risk then it is pretty evident in Algorithm~\ref{algo:approximate_greedy_m_1_rand}  that the proposed approach inherently choose the point where the direction of change in the parameter space is aligned maximally with the direction of descent on the true population risk. Let us closely look at the equation which we utilize to select our points greedily, \small $\tilde{z} \gets \argmax_{z_i \in Z \text{\textbackslash} \bar{Z}} \bbE_{z'} \nabla_{\theta} \ell(z',\hat{\theta}) H_{\hat{\theta}}^{-1}  \nabla_{\theta} \ell({z_i},\hat{\theta})$ \normalsize. 
We can write $\bbE_{z'} [\nabla_{\theta} \ell(z',\hat{\theta})] $ as gradient of true population risk $\nabla R(\theta)$ and it is clearly visible that those points are chosen with almost certainty where direction of update matches maximum with true gradient vector $\nabla R(\theta)$. 

\section{Analysis of Local Greedy Method}  \label{sec:analysis}
The algorithm discussed in the previous sections was a local greedy points selection algorithm at each stage. However,  the advantage of using local greedy selection procedure over the random sampling is unknown. In this section, we analyze the theoretical advantage of our approach over random sampling. Before going into the details of discussing the optimality guarantee of our prposed algorithm, we will define some quantities which will be helpful in quantifying the structure of the problem in rest of the section. We first define the worst point addition in a model in the following way:
\small
\balignt
\begin{split} \label{eq:worst_point_addition}
&z_{M_m+k+1}^w = \argmax_{\{\tilde{z}_{j_1},\tilde{z}_{j_2}\cdots \tilde{z}_{j_m}\} \in Z \text{\textbackslash} \bar{Z}}    R\left[ \argmin_{\theta} \left(   \sum_{i=1}^{M+m*k} \ell(z_i,\theta)    + \sum_{p=1}^{m}  \ell(z_{j_p},\theta)\right) \right] \\
&\theta_{{M_m+k+1^w}}  = \argmin_{\theta}  \left(   \sum_{i=1}^{M+m*k} \ell(z_i,\theta)    + \sum_{p=1}^{m}  \ell(z_{j_p},\theta)\right)
\end{split}
\ealignt
\normalsize
In equation~\eqref{eq:worst_point_addition},  $\theta_{{M_m+k+1^w}}$ denotes that the earlier classifier was trained on $M+m*k$ number of points and afterwards a set of  $m$ points were added which decrease the risk least. 
We further define two more quantity  which quantify the instantaneous gain/loss of choosing greedy points over averaged random points. 
\small
\balignt
\Delta_{\theta_{M_m+k}}' &= R(\theta_{M_m+k+1^g}) - \bbE [ R(\theta_{M_m+k+1^r}) ]  \label{eq:delta_1_greedy_rand} \text{ and }
\Delta_{\theta_{M_m+k}}'' &= R(\theta_{M_m+k+1^g}) - R(\theta_{M_m+k+1^w}) \label{eq:delta_1_greedy_worse} 
\ealignt
\normalsize
$\Delta'$ represents the quality of local greedy search over random selection of point and $\Delta''$ represents the gap between the best point chosen greedily and worst possible point in the same step. It is clear that the following relation holds:  \small $ \Delta_{\theta_{M+k}}'' \leq \Delta_{\theta_{M+k}}' \leq 0 ~~\forall ~\theta_{M+k} \in \Theta $\normalsize. \\
From the approximation given in equation~\eqref{eq:loss_m_points}, we can easily observe that the order of $\Delta''$ and $\Delta'$ is roughly of the order of $\mathcal{O}\left( \frac{m^2}{(M+m)^2} \right) \in  \mathcal{O}\left( \frac{1}{M} \right)$ for small enough $m$. Now we start by analyzing the optimality guarantee of our proposed greedy algorithm for the case of $m=1$ and later will  state result for general $m$.  A trivial result about comparison between the output of the algorithm~\ref{algo:exact_greedy_m} and \ref{algo:approximate_greedy_m_1_rand} is given in lemma~\ref{lem:exact_vs_approx}.  The proofs of all results in this section can be found in Appendix~\ref{ap:analysis}.
\begin{lemma}\label{lem:exact_vs_approx}
Let us assume that the current parameter $\theta_{M_m+k}$ is obtained by training on $M+k*m$ number of points. If $\tilde{\theta}_{M_m+k+1^g}$ and $\hat{\theta}_{M_m+k+1^g}$  are the output of Algorithm~\ref{algo:exact_greedy_m} and Algorithm~\ref{algo:approximate_greedy_m_1_rand} correspondingly after adding set of $m$ points greedily then under the bounded gradient and bounded Hessian assumption on the function of interest $R:\Theta\rightarrow \mathbb{R}$ , the following holds:
\balignt
|R(\tilde{\theta}_{M_m+k+1^g})  - R(\hat{\theta}_{M_m+k+1^g})| \in \mathcal{O}\left(\frac{m^2}{(M+(k+1)m)^2} \right).
\ealignt
\end{lemma}
Next result we provide for $m = 1$ where we compare the gain of greedy subset selection over randomized selection in two consecutive steps which later we will unroll to get the general result. 
\begin{lemma} \label{lem:m_1_2_cons_compare}
Under the assumptions discussed in the section~\ref{subsec:assumptions}, the gain in $ R $ after two rounds of greedy point selection over two rounds of random selection provided that initial given classifier is $\theta_{M_m}$ which has been trained on $M$ point, can be given as following:
\balignt
R(\tilde{\theta}_{M_m+{2}^g}) - \bbE[ R(\hat{\theta}_{M_m+{2}^r})] \leq \Delta_{\theta_{M_m+1^g}}' + \Delta_{\theta_{M}}'  + \frac{G'm^2}{(M+m)^2},
\ealignt
for some real positive constant $G'$.
\end{lemma}
 Lemma~\ref{lem:m_1_2_cons_compare} basically tells that greedy selection has advantage over the random selection after two steps. Now we combine the result from the lemma~\ref{lem:exact_vs_approx} and lemma~\ref{lem:m_1_2_cons_compare} to provide a result on the gain of greedy selection over random selection over $2p$ successive steps.
 \begin{theorem} \label{thm:optimality_m_point_addition}
 If we run our greedy algorithm for $p$ successive steps and all the assumptions discussed in the section~\ref{subsec:assumptions} are satisfied then the gain in $ R $ after $p$ successive greedy steps  over the same number of random selection starting from the classifier $\theta_M$ can be characterized as follows:
 \balignt
 R(\tilde{\theta}_{M_m+{p}^g}) -  \bbE[ R(\hat{\theta}_{M_m+{p}^r})] \leq \sum_{i = 0}^{p-1}\Delta_{\theta_{M_m+i^g}}'  + \sum_{i = 1}^{p-1}\frac{\tilde{G}m^2}{(M+i*m)^2},
 \ealignt
 for some positive constant $\tilde{G}$ where $\tilde{\theta}_{M+{p}^g}$ is the output from the approximate algorithm~\ref{algo:approximate_greedy_m_1_rand}.
 \end{theorem}
\small
\begin{remark}
In the above theorem for $m=1$, the term $\Delta_{\theta_{M+i^g}}'$ is roughly in the order of $\mathcal{O} \left( \frac{1}{M+i}\right)$ and it is negative because it signifies the decrease in risk at each step. Hence, more often than not  $\Delta_{\theta_{M+i^g}}'$ would dominate the term  $\frac{\tilde{G}}{(M+i)^2}$ whenever there are influential points remaining in the set. However, $\Delta_{\theta_{M+i^g}}'$ would decrease over time due to two reasons: (i) Because $\frac{1}{M+i}$ goes down as $i$ increases. (ii) Because number of influential points decreases as you select more and more points.  
\end{remark}
\normalsize

However, if we consider $m=1$, our method does not have any optimal guarantee with respect to the best subset of size $M+i$ chosen in $i_{th}$ iteration but here below we provide a small result which says that it might be still reasonable thing to do when someone does not want to do combinatorial  search in each iteration. Let us first denote the notation for the next theorem. $R^\star_M$ denotes the optimal risk for the the best set of $M$ points from the training set and $\hat{Z}_M$ denotes the best subset. $\hat{Z}_{M\backslash z^w}$ denotes removing worst performing element from the set and then $\hat{Z}_{M\backslash z^w \cup z}$ denotes adding $z$ afterwards.
\begin{lemma} \label{lem:compare_optimal}
If the initial model is represented with $\theta_M$ and $R(\theta_M) \leq \nu R^\star_M + \delta$, then 
\balignt
R(\theta_{M+1^g}) \leq \nu R^\star_M + \delta + \Delta''_{\theta_M}  + \mathcal{O} \left(  \frac{1}{M+1}\right)
\ealignt
\end{lemma}
Similar to the previous theorem, the result can be obtained for a larger number of iterations.
\small
\begin{remark}
In the above result the term $\Delta''_{\theta_M} $ is negative which is the difference in the function value of $R$ after adding best and worst point in the training set. Hence this term is supposed to have high magnitude however the other order $\mathcal{O}(\frac{1}{M+1})$ comes from the difference. 
\end{remark}
\normalsize


\section{Experiments} \label{sec:expts}
We now compare our subset selection strategy with random sampling on a variety of classification and regression tasks. For the classification tasks we use the Amazon.com employee access Kaggle competition dataset \cite{liu2017mlbench} (32,769 sample points, 135 features, 2 classes) and the MNIST handwritten digits dataset (70,000 sample points, 784 features, 10 classes). For regression, we use the Boston housing dataset \cite{harrison1978hedonic} (506 sample points, 13 features) and the California housing dataset \cite{pace1997sparse} (20,640 sample points, 8 features).

We evaluate our method in three settings: (i) using the same model to subsample the data and to evaluate performance; (ii) using different models; (iii) using different models as part of a hyperparameter tuning task. Our initial model was trained on $d$ number of datapoints where $d$ represents the dimension. 

\begin{figure}
  \begin{subfigure}{.33\textwidth}
    \centering
    \includegraphics[width=.98\linewidth]{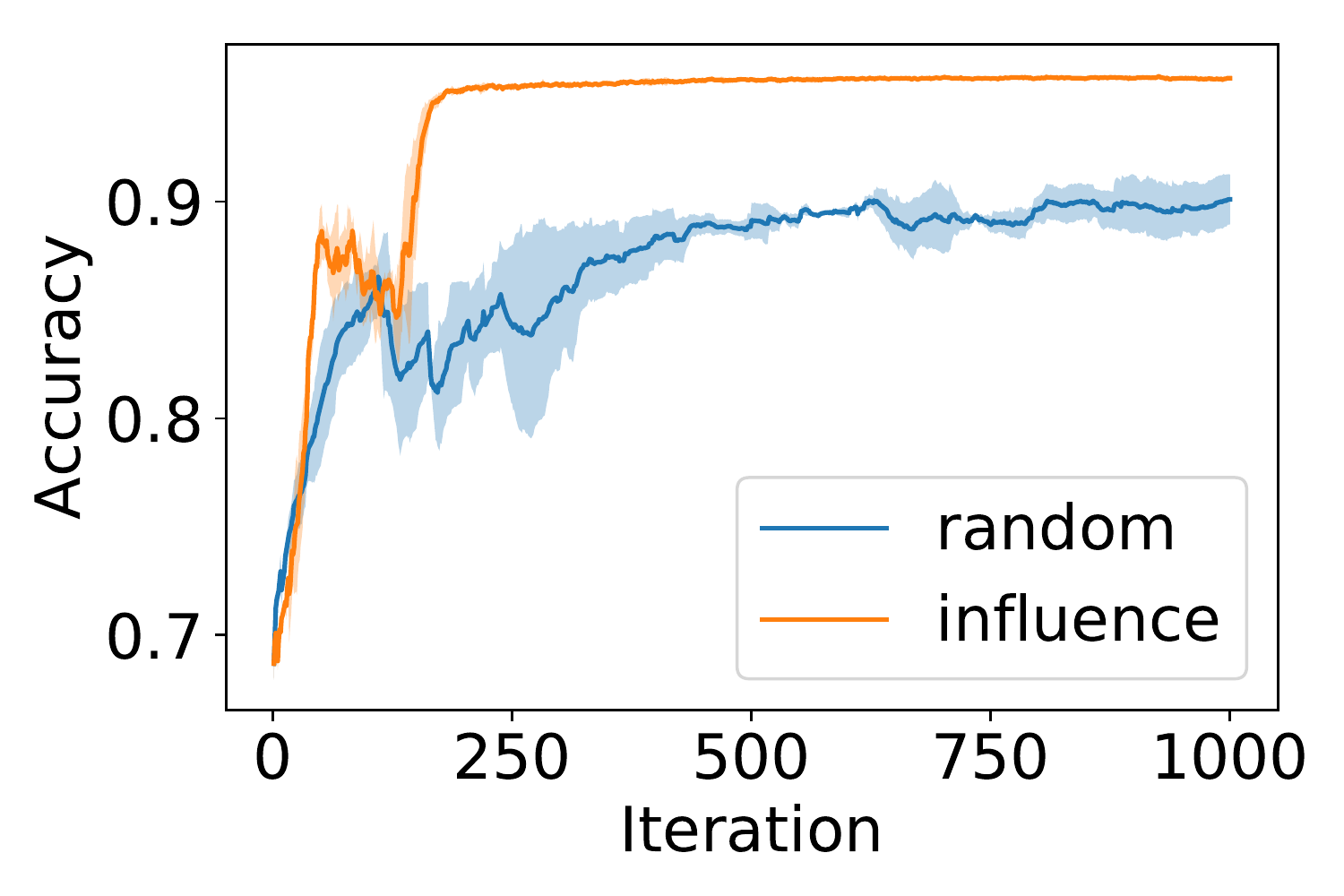}
    \caption{Amazon (Logistic regression) }
    \label{fig:select_fig1}
  \end{subfigure}%
  \hfill
  \begin{subfigure}{.33\textwidth}
    \centering
    \includegraphics[width=.98\linewidth]{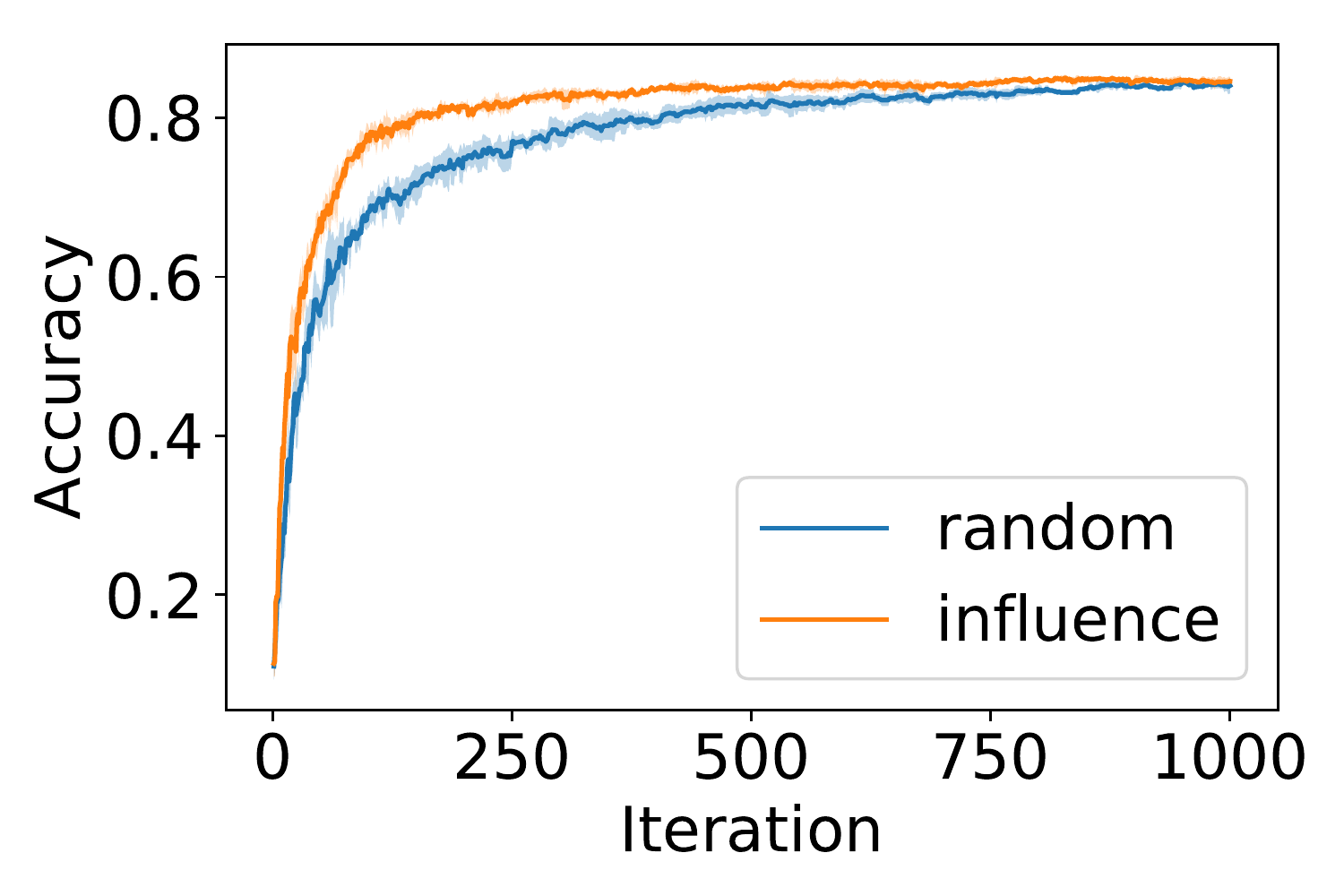}
    \caption{MNIST (Logistic regression)}
    \label{fig:select_fig2}
  \end{subfigure}
  \hfill
  \begin{subfigure}{.33\textwidth}
    \centering
    \includegraphics[width=.98\linewidth]{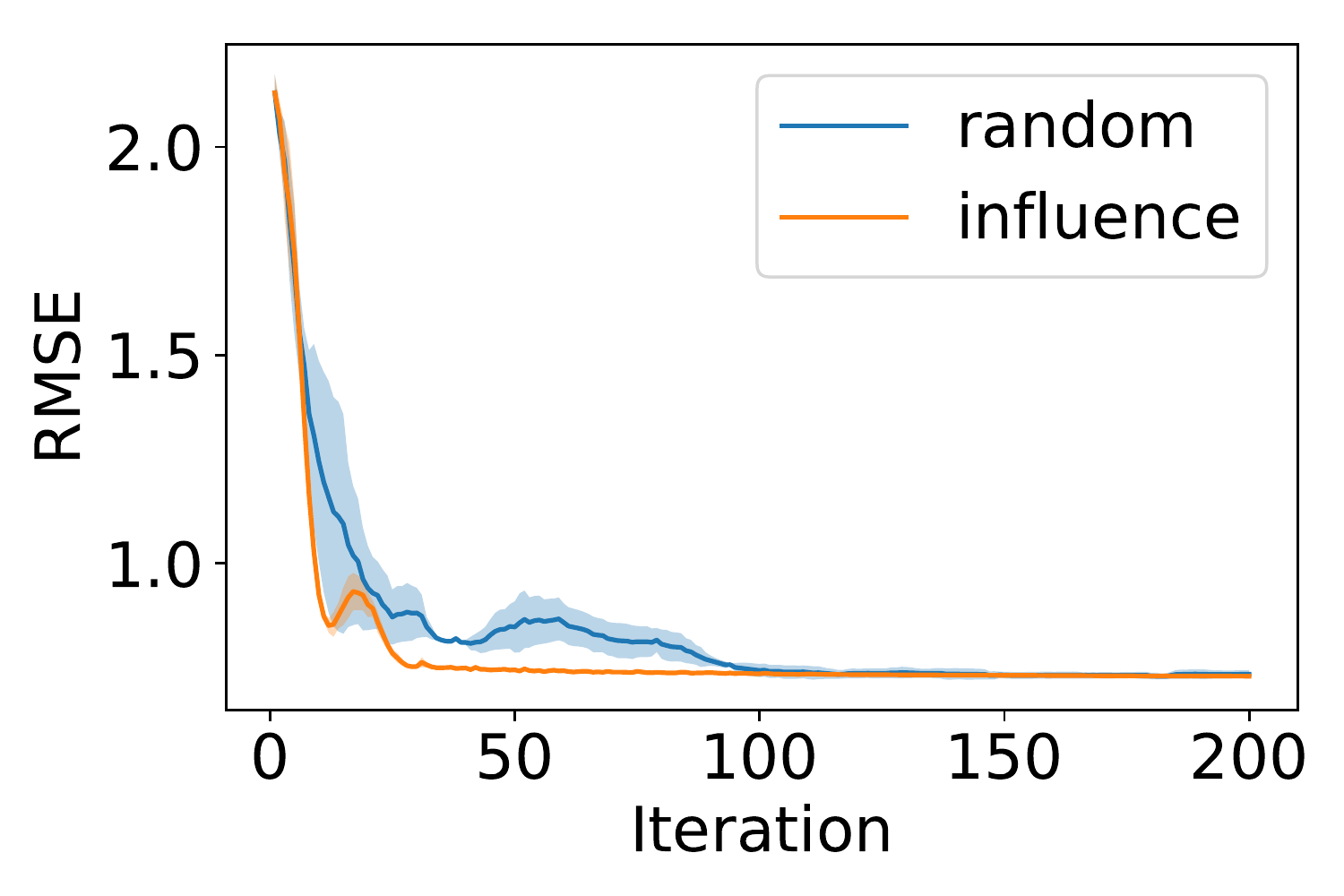}
    \caption{California (Linear regression)}
    \label{fig:select_fig3}
  \end{subfigure}%

  \begin{subfigure}{.33\textwidth}
    \centering
    \includegraphics[width=.99\linewidth]{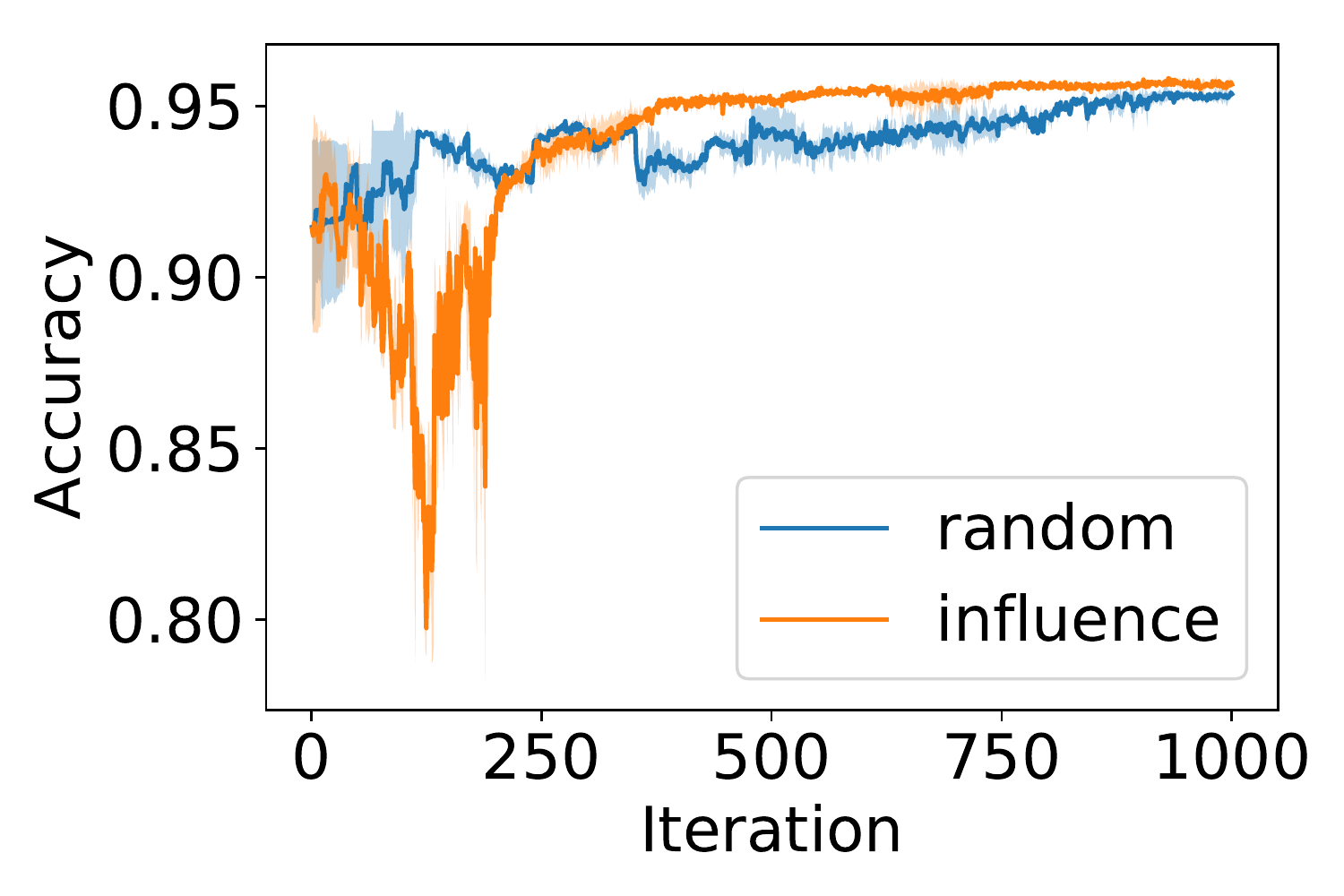}
    \caption{Amazon (LGBM) Selection}
    \label{fig:LGBT_select}
  \end{subfigure}%

 \begin{subfigure}{.33\textwidth}
    \centering
    \includegraphics[width=.99\linewidth]{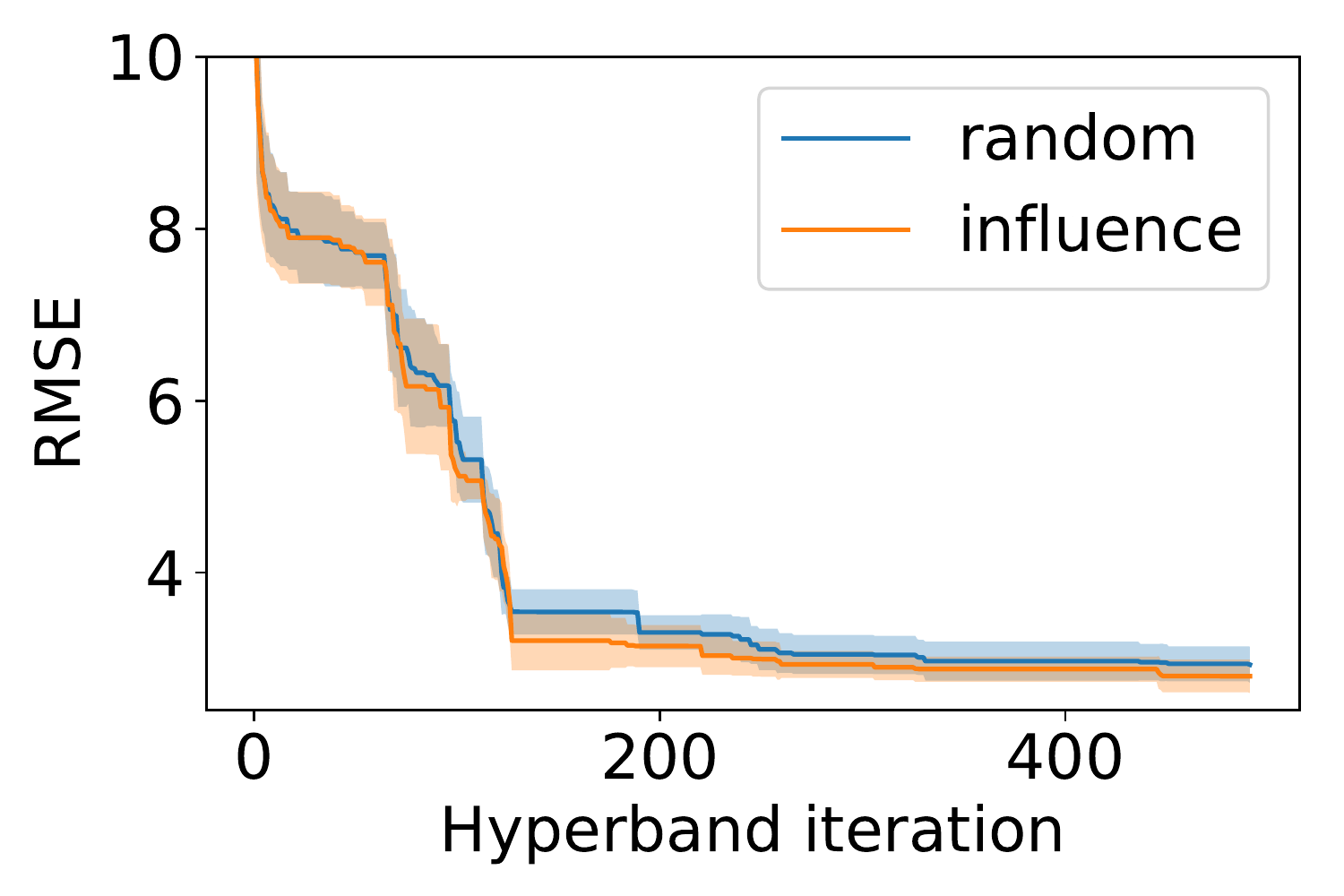}
    \caption{Boston Hyperband (RMSE)}
    \label{fig:boston_hband_iter}
  \end{subfigure}%
  \hfill
  \begin{subfigure}{.33\textwidth}
    \centering
    \includegraphics[width=.99\linewidth]{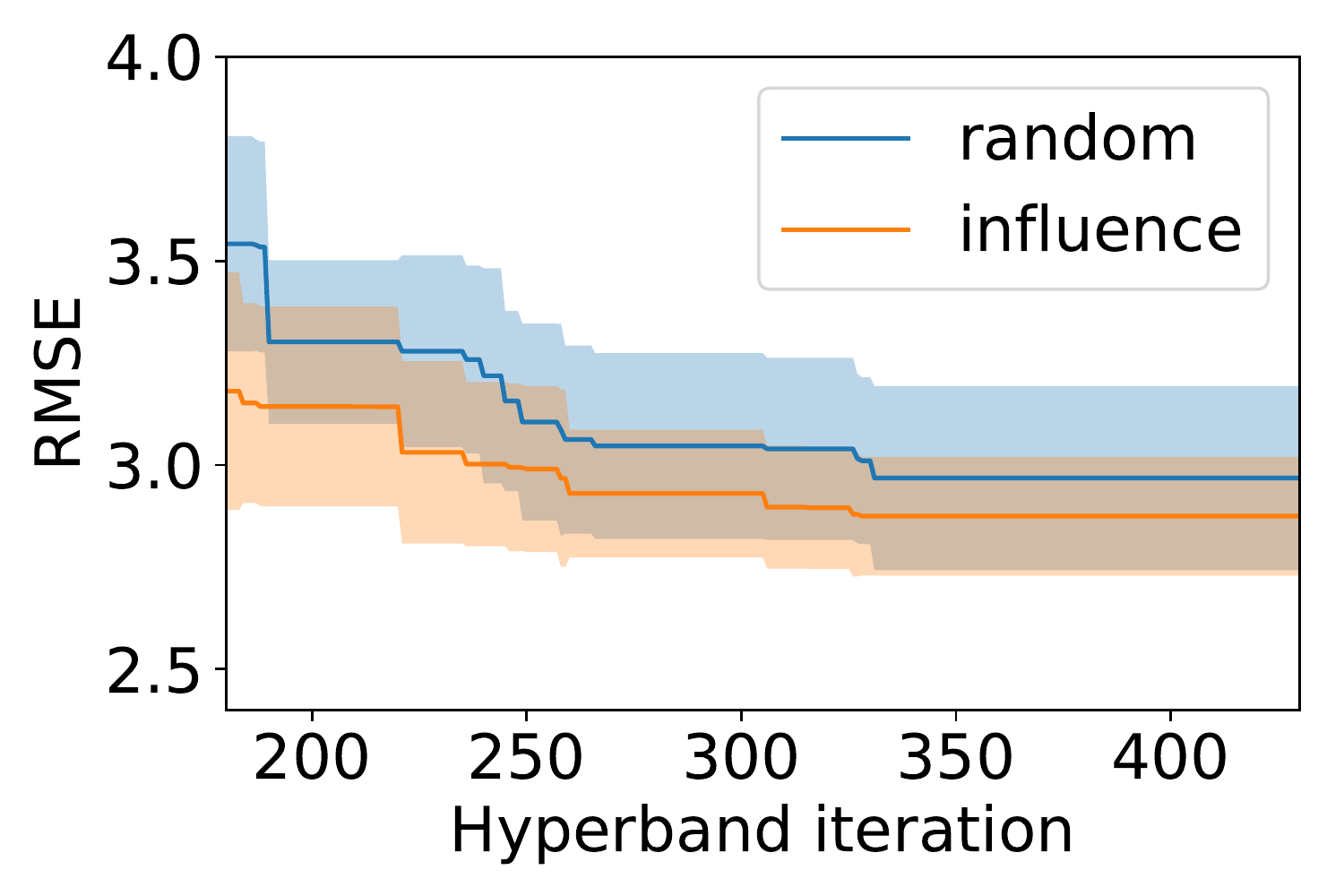}
    \caption{Boston Hyperband (Zoomed in)}
    \label{fig:zoom_boston1}
  \end{subfigure}%
\hfill
\begin{subfigure}{.33\textwidth}
    \centering
    \includegraphics[width=.99\linewidth]{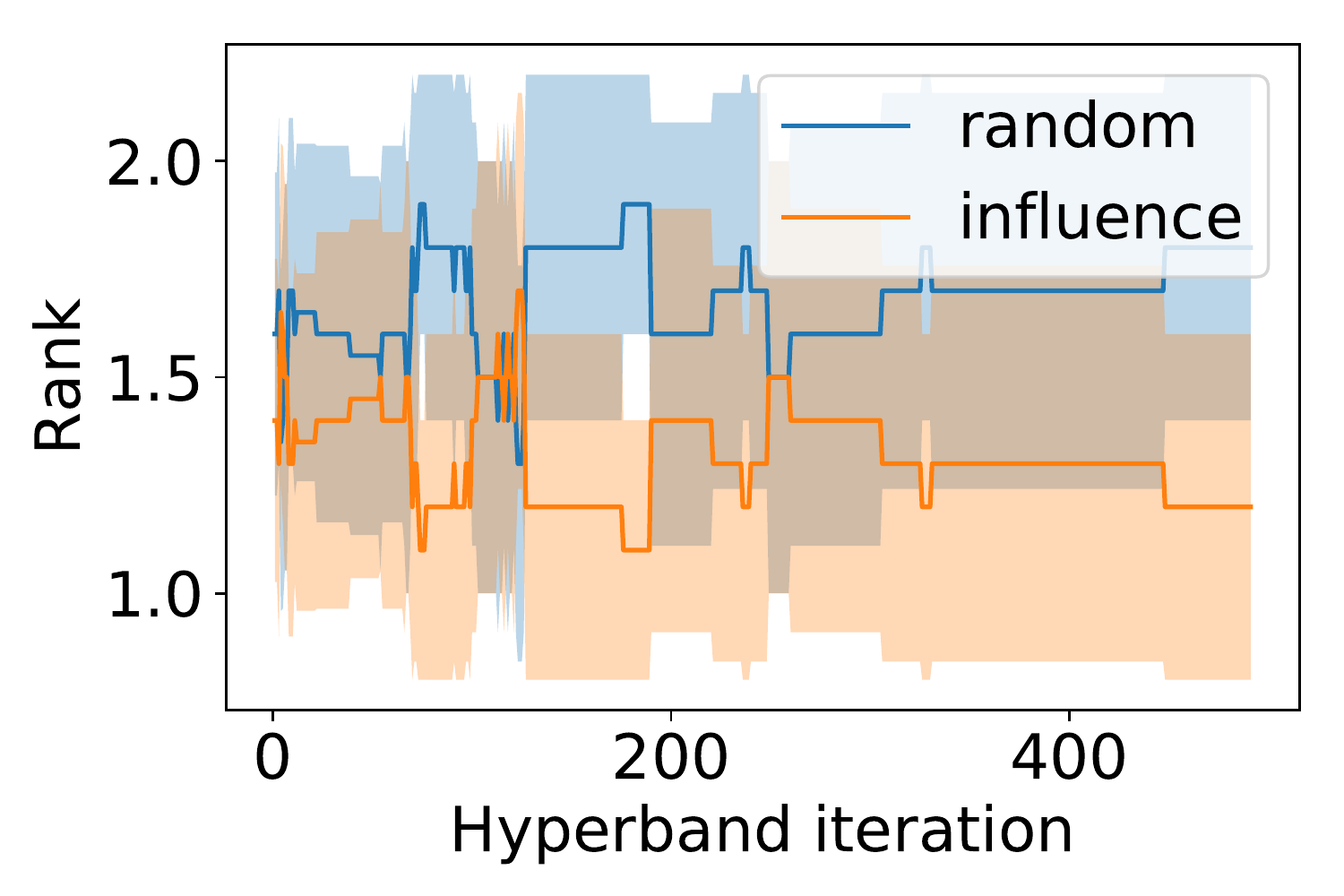}
    \caption{Boston Hyperband (Rank)}
    \label{fig:boston_hband_rank}
  \end{subfigure}%

  \caption{Experimental results (see \cref{sec:expts})} 
  \label{fig:selectfig}
  \end{figure}

\paragraph{Rapid assessment of model performance.} For this task, we consider the \textit{Amazon}, \textit{MNIST} and \textit{California} housing datasets. We run subsample selection using logistic regression for the classification tasks and linear regression for the regression task. At each iteration we add $m=1$ sample point, running for 1000 iterations for \textit{Amazon} and \textit{MNIST} and for 200 iteration for \textit{California}. For each dataset, we hold out 20\% of the data as a test set and we further split the remaining data in a training set (80\%) and a validation set (20\%). 
We run each experiment 10 times with different random seeds (0 through 9) and report the mean and standard deviation across seeds. As shown in Figures~\ref{fig:select_fig1},~\ref{fig:select_fig2} and \ref{fig:select_fig3} our subsampling approach significantly outperforms random selection in all the tasks we consider.
\paragraph{Model transfer: rapid assessment of other models' performance.} The goal of this experiment is to directly show that selected subsets of the samples using our proposed greedy approach for a fixed model can generalize to other models as well. For this experiment, we consider the \textit{Boston} dataset and follow the experimental procedure described above. We use a linear regression model to subsample the data and a gradient boosted regression tree (GBRT) to evaluate performance. We use the LightGBM implementation of GBRTs with default parameters.
As shown in Figure~\ref{fig:LGBT_select}, our subsampling approach outperforms random selection even when the models used to select the subsamples and to analyze the data are different (and in this case, the latter is non-differentiable).
\paragraph{Improving hyperparameter tuning.} The goal of this experiment is to assess whether improvements in subsample selection lead to improvements in hyperparameter tuning tasks that rely on subsampling. For this experiment, we consider the \textit{Boston} dataset and tune the hyperparameters of a random forest. Specifically we, tune the number of trees (5 to 20), the maximum number of features (1\% to 100\%), the minimum number of samples on which to split (2 to 11), the minimum number of samples on a leaf (2 to 11) and whether to bootstrap the trees or not. We run Hyperband~\cite{li2017hyperband} with default parameters cycling through $\eta \in [2, 5]$. Figure~\ref{fig:boston_hband_iter} shows the root mean squared error (RMSE) as a function of the number of hyperband iterations and Figure~\ref{fig:boston_hband_rank} shows the relative ranks of the two subsampling methods we consider (lower is better). In Figure~\ref{fig:zoom_boston1}, we showed the zoomed version of Figure~\ref{fig:boston_hband_iter} to closely look for the gain. Our approach consistently outperforms the random sampling approach from the very first few Hyperband iterations.

\section{Conclusion and Future Work}
In this paper, we have proposed an efficient procedure for greedy selection of subsets of training points from large datasets in an empirical risk minimization setting. A promising future direction  is to extend our work to a streaming data setting, where points must be selected in an online manner. It would be also interesting to extend our work to subset selection for multiple models, showing that we can select subsets that work well for different models being trained and compared on the same dataset.

 ﻿\bibliographystyle{abbrvnat}
\bibliography{subset_influence}

\begin{thebibliography}{25}
\providecommand{\natexlab}[1]{#1}
\providecommand{\url}[1]{\texttt{#1}}
\expandafter\ifx\csname urlstyle\endcsname\relax
  \providecommand{\doi}[1]{doi: #1}\else
  \providecommand{\doi}{doi: \begingroup \urlstyle{rm}\Url}\fi

\bibitem[Agarwal et~al.(2017)Agarwal, Bullins, and Hazan]{agarwal2017second}
N.~Agarwal, B.~Bullins, and E.~Hazan.
\newblock Second-order stochastic optimization for machine learning in linear
  time.
\newblock \emph{The Journal of Machine Learning Research}, 18\penalty0
  (1):\penalty0 4148--4187, 2017.

\bibitem[Bachem et~al.(2017)Bachem, Lucic, and Krause]{bachem2017practical}
O.~Bachem, M.~Lucic, and A.~Krause.
\newblock Practical coreset constructions for machine learning.
\newblock \emph{arXiv preprint arXiv:1703.06476}, 2017.

\bibitem[Campbell and Broderick(2017)]{campbell2017automated}
T.~Campbell and T.~Broderick.
\newblock Automated scalable bayesian inference via hilbert coresets.
\newblock \emph{arXiv preprint arXiv:1710.05053}, 2017.

\bibitem[Christmann et~al.(2007)Christmann, Steinwart,
  et~al.]{christmann2007consistency}
A.~Christmann, I.~Steinwart, et~al.
\newblock Consistency and robustness of kernel-based regression in convex risk
  minimization.
\newblock \emph{Bernoulli}, 13\penalty0 (3):\penalty0 799--819, 2007.

\bibitem[Clarkson and Woodruff(2015)]{clarkson2015sketching}
K.~L. Clarkson and D.~P. Woodruff.
\newblock Sketching for {M}-estimators: A unified approach to robust
  regression.
\newblock In \emph{\SODA{2015}}, pages 921--939, 2015.

\bibitem[Cohen et~al.(2015)Cohen, Lee, Musco, Musco, Peng, and
  Sidford]{cohen2015uniform}
M.~B. Cohen, Y.~T. Lee, C.~Musco, C.~Musco, R.~Peng, and A.~Sidford.
\newblock Uniform sampling for matrix approximation.
\newblock In \emph{Proceedings of the 2015 Conference on Innovations in
  Theoretical Computer Science}, pages 181--190. ACM, 2015.

\bibitem[Cohen et~al.(2017)Cohen, Musco, and Musco]{cohen2017input}
M.~B. Cohen, C.~Musco, and C.~Musco.
\newblock Input sparsity time low-rank approximation via ridge leverage score
  sampling.
\newblock In \emph{\SODA{2017}}, 2017.

\bibitem[Drineas et~al.(2006)Drineas, Mahoney, and
  Muthukrishnan]{drineas2006sampling}
P.~Drineas, M.~W. Mahoney, and S.~Muthukrishnan.
\newblock Sampling algorithms for $\ell_2$ regression and applications.
\newblock In \emph{\SODA{2006}}, 2006.

\bibitem[Falkner et~al.(2018)Falkner, Klein, and Hutter]{falkner2018bohb}
S.~Falkner, A.~Klein, and F.~Hutter.
\newblock Bohb: Robust and efficient hyperparameter optimization at scale.
\newblock \emph{arXiv preprint arXiv:1807.01774}, 2018.

\bibitem[Feldman et~al.(2007)Feldman, Monemizadeh, and Sohler]{feldman2007ptas}
D.~Feldman, M.~Monemizadeh, and C.~Sohler.
\newblock A ptas for k-means clustering based on weak coresets.
\newblock In \emph{Proceedings of the twenty-third annual symposium on
  Computational geometry}, pages 11--18. ACM, 2007.

\bibitem[Giordano et~al.(2018)Giordano, Stephenson, Liu, Jordan, and
  Broderick]{giordano2018return}
R.~Giordano, W.~Stephenson, R.~Liu, M.~I. Jordan, and T.~Broderick.
\newblock Return of the infinitesimal jackknife.
\newblock \emph{arXiv preprint arXiv:1806.00550}, 2018.

\bibitem[Hampel et~al.(2011)Hampel, Ronchetti, Rousseeuw, and
  Stahel]{hampel2011robust}
F.~R. Hampel, E.~M. Ronchetti, P.~J. Rousseeuw, and W.~A. Stahel.
\newblock \emph{Robust statistics: the approach based on influence functions},
  volume 196.
\newblock John Wiley \& Sons, 2011.

\bibitem[Har-Peled and Mazumdar(2004)]{har2004coresets}
S.~Har-Peled and S.~Mazumdar.
\newblock On coresets for k-means and k-median clustering.
\newblock In \emph{Proceedings of the thirty-sixth annual ACM symposium on
  Theory of computing}, pages 291--300. ACM, 2004.

\bibitem[Harrison~Jr and Rubinfeld(1978)]{harrison1978hedonic}
D.~Harrison~Jr and D.~L. Rubinfeld.
\newblock Hedonic housing prices and the demand for clean air.
\newblock \emph{Journal of environmental economics and management}, 5\penalty0
  (1):\penalty0 81--102, 1978.

\bibitem[Huggins et~al.(2016)Huggins, Campbell, and
  Broderick]{huggins2016coresets}
J.~Huggins, T.~Campbell, and T.~Broderick.
\newblock Coresets for scalable bayesian logistic regression.
\newblock In \emph{Advances in Neural Information Processing Systems}, pages
  4080--4088, 2016.

\bibitem[Koh and Liang(2017)]{koh2017understanding}
P.~W. Koh and P.~Liang.
\newblock Understanding black-box predictions via influence functions.
\newblock \emph{arXiv preprint arXiv:1703.04730}, 2017.

\bibitem[Li et~al.(2017)Li, Jamieson, DeSalvo, Rostamizadeh, and
  Talwalkar]{li2017hyperband}
L.~Li, K.~Jamieson, G.~DeSalvo, A.~Rostamizadeh, and A.~Talwalkar.
\newblock Hyperband: A novel bandit-based approach to hyperparameter
  optimization.
\newblock \emph{The Journal of Machine Learning Research}, 18\penalty0
  (1):\penalty0 6765--6816, 2017.

\bibitem[Liu et~al.(2017)Liu, Zhang, Zeng, Wu, and Zhang]{liu2017mlbench}
Y.~Liu, H.~Zhang, L.~Zeng, W.~Wu, and C.~Zhang.
\newblock Mlbench: How good are machine learning clouds for binary
  classification tasks on structured data?
\newblock \emph{arXiv preprint arXiv:1707.09562}, 2017.

\bibitem[Lucic et~al.(2017)Lucic, Faulkner, Krause, and
  Feldman]{lucic2017training}
M.~Lucic, M.~Faulkner, A.~Krause, and D.~Feldman.
\newblock Training mixture models at scale via coresets.
\newblock \emph{stat}, 1050:\penalty0 23, 2017.

\bibitem[Mahoney(2011)]{mahoney2011randomized}
M.~W. Mahoney.
\newblock Randomized algorithms for matrices and data.
\newblock \emph{Foundations and Trends in Machine Learning}, 3\penalty0
  (2):\penalty0 123--224, 2011.

\bibitem[Martens(2010)]{martens2010deep}
J.~Martens.
\newblock Deep learning via hessian-free optimization.
\newblock In \emph{ICML}, volume~27, pages 735--742, 2010.

\bibitem[Munteanu et~al.(2018)Munteanu, Schwiegelshohn, Sohler, and
  Woodruff]{munteanu2018coresets}
A.~Munteanu, C.~Schwiegelshohn, C.~Sohler, and D.~P. Woodruff.
\newblock On coresets for logistic regression.
\newblock In \emph{\NIPS{2018}}, 2018.

\bibitem[Pace and Barry(1997)]{pace1997sparse}
R.~K. Pace and R.~Barry.
\newblock Sparse spatial autoregressions.
\newblock \emph{Statistics \& Probability Letters}, 33\penalty0 (3):\penalty0
  291--297, 1997.

\bibitem[Pilanci and Wainwright(2016)]{pilanci2016iterative}
M.~Pilanci and M.~J. Wainwright.
\newblock Iterative hessian sketch: Fast and accurate solution approximation
  for constrained least-squares.
\newblock \emph{The Journal of Machine Learning Research}, 17\penalty0
  (1):\penalty0 1842--1879, 2016.

\bibitem[Ting and Brochu(2018)]{ting2018optimal}
D.~Ting and E.~Brochu.
\newblock Optimal subsampling with influence functions.
\newblock In \emph{Advances in Neural Information Processing Systems}, pages
  3654--3663, 2018.

\end{thebibliography}

\newpage
\onecolumn
\appendix
\begin{center}
{\centering \LARGE Appendix }
\vspace{1cm}
\sloppy

\end{center}

Before moving to proof of the  technical theorems and Lemmas, we reiterate our previous notations as well as define some new notations. 

\paragraph{Notations:} A subset of $M$ data points $\{z_i\}_{i=1}^M$ are being expanded. An additional point selected using proposed greedy approach and using uniform random sampling is denoted by $z_{M+1}^g$  and $z_{M+1}^r$  respectively.  The point mass distribution on this new point is $Q_{M}^g =  \delta_{z_{M+1}^g}$ or $Q_{M}^r =  \delta_{z_{M+1}^r}$ respectively. 

Let $P^M$ be the empirical distribution on $\{z_i\}_{i=1}^M$. 
The new distribution  after adding a point is \small $P^M_{\epsilon_M,Q_M^t} = \frac{1}{M+1} \sum_{i = 1}^{M}\delta_{z_i} + \epsilon_M \delta_{z_{M+1}^t}$ \normalsize{for} \small $t \in \{ r,g\}$ \normalsize where \small$\epsilon _M = \frac{1}{M+1}$\normalsize. So, \small $P^M_{\epsilon_M, Q_M^t} = (1 - \epsilon_M) P^M + \epsilon_M Q_M^t$\normalsize.

For adding $m$ new points at a time, we denote one block of points as \small$z^t_{M_m+1}$ \normalsize which consists $[z^t_{M+i}]_{i =1}^m$ where \small$t \in \{r,g\}$\normalsize. We write \small$Q_{M_m}^t = \frac{1}{m} \sum_{i = 1}^{m}\delta_{z_{M+i}^t} $ \normalsize and \small $P^M_{\epsilon_{M_m},Q_{M_m}^t} = \frac{1}{M+m} \sum_{i =1}^{M}\delta_{z_{i}} + \frac{1}{M+m} \sum_{i =1}^m \delta_{z_{M+i}^t} $\normalsize. We can see that \small$P^M_{\epsilon_{M_m}, Q_M^t} = (1 - \epsilon_{M_m}) P^M + \epsilon_{M_m} Q_{M_m}^t$\normalsize where \small $\epsilon_{M_m} = \frac{m}{M+m}$\normalsize. 

 \small $ \hat{\theta}_M \eqdef \hat{\theta}(P^M)$ \normalsize denotes the initial model trained on \small $\{z_i\}_{i=1}^M$\normalsize and  \small $\hat{\theta}_{M+1^t} = \hat{\theta}(P^M_{\epsilon_M,Q_M^t})$\normalsize ~denote the model trained on \small$P^M_{\epsilon_M,Q_M^t}$\normalsize. We extend this notation, e.g., letting \small$\hat{\theta}_{M+k_1^g+k_2^r+k_3^g}$\normalsize ~ denote the model which is trained by first adding $k_1$ points greedily in sequence, starting at the model \small$\hat{\theta}_{M}$\normalsize, then adding $k_2$ points randomly at  \small$\hat{\theta}_{M+k_1^g}$\normalsize ~ and finally adding $k_3$ points greedily starting at \small $\hat{\theta}_{M+k_1^g+k_2^r}$. $\hat{\theta}_{M_m+k_1^g+k_2^r+k_3^g}$\normalsize ~ can be defined similar to \small $\hat{\theta}_{M+k_1^g+k_2^r}$ \normalsize  where single point is replaced with the block of $m$ points.

\section{First Order Model Approximation} \label{ap:foa}

\begin{lemma}[Lemma~\ref{lem:1_point_addition_foa}]
Let $\tilde{\theta}_{M+1^t} = \hat{\theta}_M + \frac{1}{M+1} BIF(Q_M^t;\hat{\theta},P^M)$ be the first order approximation of the estimator $\hat{\theta}_{M+1^t} $ as in \eqref{eq:first_order_approx_T}. Under Assumptions~\ref{as:smooth}, \ref{as:non-degenerate}, \ref{as:bounded},  and \ref{as:local_smooth}, $BIF(Q_M^t;\hat{\theta},P^M)$ is given as:
\balignt
BIF(Q_M^t;\hat{\theta},P^M) =  - {H_{\hat{\theta}_M} ^{-1}} \nabla_{\theta} \ell(z_{M+1}^t,\hat{\theta}_M) 
\ealignt
and $\| \tilde{\theta}_{M+1^t} -\hat{\theta}_{M+1^t} \| \leq \frac{B_1}{(M+1)^2}$ for some constant $B_1$.
\end{lemma}
\begin{proof}
Recall that $\hat{\theta}_{M+1^t} \eqdef \hat{\theta}(P^M_{\epsilon_M, Q_M^t})$ is given by:
\balignt
\hat{\theta}_{M+1^t} =\argmin_{\theta \in \Omega_\theta} \left(\frac{1}{M+1} \sum_{j =1}^{M} \ell({z}_j,\theta) + \frac{1}{M+1}\ell({z}_{M+1}^t,\theta) \right).
\ealignt
We can write:
\balignt
BIF(Q_M^t;\hat{\theta},P^M)  = \lim_{\epsilon \rightarrow 0} \frac{\hat{\theta}\left( (1-\epsilon) P^M + \epsilon Q_M^t \right) - \hat{\theta}(P^M)}{\epsilon}.
\ealignt
Using Assumption \ref{as:local_smooth}, that $\ell$ is locally smooth, we have for any distribution $P$:
\balignt
\nabla_{\theta} \left( \sum_{j =1}^M p_i \ell({z}_j,\theta) \right)\Bigg|_{\theta= \hat{\theta}(P)}  =  \sum_{j =1}^M p_i \cdot \nabla_{\theta} \ell({z}_j,\theta)\big|_{\theta = \hat{\theta}(P)} = 0.
\ealignt

Letting $\hat{\theta}_{\epsilon,z}$ denote $\hat{\theta}\left( (1-\epsilon) P^M + \epsilon Q_M^t \right)$ (see definition in Problem \ref{prob:subsamp}), we thus have:
\balignt
\frac{(1-\epsilon)}{M}\sum_{i =1}^M \nabla_{\theta} ~\ell (z_i, \hat{\theta}_{\epsilon,z} ) + \epsilon ~\nabla_{\theta} \ell(z_{M+1}^t,\hat{\theta}_{\epsilon,z}) = 0. \label{eq:optimal_condition}
\ealignt
Applying a Taylor expansion to equation~\eqref{eq:optimal_condition} we have:
\small
\balignt
0 &= \left[ \underbrace{\frac{(1-\epsilon)}{M}\sum_{i =1}^M \nabla_{\theta} ~\ell(z_i, \hat{\theta}_M ) }_{=0}+ \epsilon ~\nabla_{\theta} \ell(z_{M+1}^t,\hat{\theta}_M) \right] \\& \qquad \qquad \qquad  + \left[ \frac{(1-\epsilon)}{M} \sum_{i =1}^M \nabla^2_{\theta}~ \ell({z}_j,\hat{\theta}_M) + \epsilon \nabla^2_{\theta} \ell(z_{M+1}^t,\hat{\theta}_M)\right] \left(\hat{\theta}_{\epsilon,z} - \hat{\theta}_M\right) + \mathcal{O}(\| \hat{\theta}_{\epsilon,z} - \hat{\theta}_M \|^2).
\ealignt
\normalsize
Hence, 
\balignt
\hat{\theta}_{\epsilon,z} - \hat{\theta}_M =-\epsilon  \left[ \frac{(1-\epsilon)}{M} \sum_{i =1}^M \nabla^2_{\theta}~ \ell({z}_j,\hat{\theta}_M) + \epsilon \nabla^2_{\theta} \ell(z_{M+1}^t,\hat{\theta}_M)\right]^{-1} \nabla_{\theta} \ell(z_{M+1}^t,\hat{\theta}_M) + \mathcal{O}(\| \hat{\theta}_{\epsilon,z} - \hat{\theta} \|^2).
\ealignt
As $\epsilon \rightarrow 0$, $\hat{\theta}_{\epsilon,z} \rightarrow \hat{\theta}_M$.
This gives:
\balignt
BIF(Q_M^t;\hat{\theta},P^M)  = -\left[\underbrace{ \frac{1}{M} \sum_{i =1}^M \nabla^2_{\theta}~ \ell({z}_j,\hat{\theta}_M)}_{H_{\hat{\theta}_M}}\right]^{-1} \nabla_{\theta} \ell(z_{M+1}^t,\hat{\theta}_M). 
\ealignt
Hence, we can write the first order approximation $\tilde{\theta}_{M+1^t}$ of $\hat{\theta}_{M+1^t}$ as:
\balignt
\tilde{\theta}_{M+1^t}  \eqdef \hat{\theta}_M + \frac{1}{M+1} BIF(Q_M^t;\hat{\theta},P^M) = \hat{\theta}_{M} - \frac{1}{M+1} {H_{\hat{\theta}_M}^{-1}} \nabla_{\theta} \ell(z_{M+1}^t,\hat{\theta}_M).
\ealignt
It directly comes from the assumptions~\ref{as:smooth}, \ref{as:non-degenerate} and \ref{as:bounded} that   $$\| \tilde{\theta}_{M+1^t} -\hat{\theta}_{M+1^t} \| < \frac{B_1}{(M+1)^2}$$ for some constant $B_1$.

\end{proof}



\begin{lemma}[Lemma \ref{lem:m_point_addition_foa}] 
Let $\tilde{\theta}_{M_m+1^t} = \hat{\theta}_M + \frac{m}{M+m} BIF(Q_{M_m}^t;\hat{\theta},P^M)$ be the first order approximation of the estimator $\hat{\theta}_{M_m+1^t} $ as in \eqref{eq:first_order_approx_T}. Under Assumptions~\ref{as:smooth}, \ref{as:non-degenerate}, \ref{as:bounded},  and \ref{as:local_smooth}, $BIF(Q_{M_m}^t;\hat{\theta},P^M)$ is given as:
\balignt
BIF(Q_{M_m}^t;\hat{\theta},P^M) &= - \sum_{j=1}^m H_{\hat{\theta}_M}^{-1} ~\nabla_{\theta} \ell({z}_{M+j}^t,\hat{\theta}_M)
\ealignt
and $\| \tilde{\theta}_{M_m+1^t}  - \hat{\theta}_{M_m+1^t}  \| \leq \frac{B_m m^2}{(M+m)^2}$ for constant $B_m$.
\end{lemma}
\begin{proof}

Recall that $\hat{\theta}_{M_m+1^t} \eqdef \hat{\theta}(P^M_{\epsilon_{M_m}, Q_{M_m}^t})$ is given by:
\balignt
\hat{\theta}_{M_m+1^t} =\argmin_{\theta} \left(\frac{1}{M+m} \sum_{j =1}^{M} \ell({z}_j,\theta) + \frac{1}{M+m}\sum_{j=1}^m \ell({z}_{M+j}^t,\theta) \right)
\ealignt
Again by first order optimality condition, we get 
\balignt
\frac{1}{M} \sum_{j =1}^{M}\nabla_{\theta} \ell({z}_j,\hat{\theta}_{M}) =0
\ealignt
Similarly, we define:
\balignt
\hat{\theta}\left( (1-\epsilon) P + \epsilon Q_{M_m}^t \right) = \argmin_{\theta} \left( \frac{ 1- \epsilon}{M} \sum_{j =1}^M \ell({z}_j,\theta) + \epsilon  ~\sum_{j=1}^m \ell({z}_{M+j}^t,\theta)\right)
\ealignt
We denote $\hat{\theta}\left( (1-\epsilon) P + \epsilon Q_{M_m}^t \right)$ with $\hat{\theta}_{\epsilon,z}$. Hence, by first order optimality condition, we get  
\balignt
\frac{1-\epsilon}{M}\sum_{i =1}^M \nabla_{\theta} ~\ell(z_i, \hat{\theta}_{\epsilon,z} ) + \epsilon ~\sum_{j=1}^m \nabla_{\theta} \ell({z}_{M+j}^t,\hat{\theta}_{\epsilon,z})= 0 \label{eq:optimal_condition_2}
\ealignt
Using similar arguments as given for single point addition case, we ca do Taylor expansion here as well.
\balignt
0 &= \left(\frac{1-\epsilon}{M}\underbrace{\sum_{i =1}^M \nabla_{\theta} \ell(z_i, \hat{\theta}_M )}_{=0} + \epsilon ~\sum_{j=1}^m \nabla_{\theta} \ell({z}_{M+j}^t,\hat{\theta}_M) \right) \\&+ \left(\frac{1-\epsilon}{M}\sum_{i =1}^M \nabla^2_{\theta} \ell(z_i, \hat{\theta}_M ) + \epsilon ~\sum_{j=1}^m \nabla^2_{\theta} \ell({z}_{M+j}^t,\hat{\theta}_M) \right)(\hat{\theta}_{\epsilon,z} - \hat{\theta}_M) 
+ \mathcal{O}(\|\hat{\theta}_{\epsilon,z} - \hat{\theta}_M\|^2)
\ealignt
Hence, 
\balignt
\hat{\theta}_{\epsilon,z} - \hat{\theta}_M =- \epsilon\left[ \frac{1-\epsilon}{M}\sum_{i =1}^M \nabla^2_{\theta} \ell(z_i, \hat{\theta}_M ) + \epsilon ~\sum_{j=1}^m \nabla^2_{\theta} \ell({z}_{M+j}^t,\hat{\theta}_M) \right]^{-1}\left( \sum_{j=1}^m \nabla_{\theta} \ell({z}_{M+j}^t,\hat{\theta}_M)  \right) \\ + \mathcal{O}(\|\hat{\theta}_{\epsilon,z} - \hat{\theta}_M\|^2)
\ealignt
Again using same argument, as $\epsilon \rightarrow 0$, $\hat{\theta}_{\epsilon,z} \rightarrow \hat{\theta}$. It essentially gives:
\balignt
BIF(Q_{M_m}^t;\hat{\theta},P^M)  \approx-\underbrace{ \left[ \frac{1}{M} \sum_{i =1}^M  \nabla^2_{\theta} \ell(z_i, \hat{\theta}_M )  \right]^{-1}}_{= H_{\hat{\theta}_M}^{-1}}\left( \sum_{j=1}^m \nabla_{\theta} \ell({z}_{M+j}^t,\hat{\theta}_M)  \right) \\ 
= -H_{\hat{\theta}_M}^{-1}\left( \sum_{j=1}^m \nabla_{\theta} \ell({z}_{M+j}^t,\hat{\theta}_M)  \right)
\ealignt
It is also evident that 
\balignt
BIF(Q_{M_m}^t;\hat{\theta},P^M) = -\sum_{j=1}^m \left(  H_{\hat{\theta}_M}^{-1} \nabla \ell({z}_{M+j}^t,\hat{\theta}_M)  \right) = \sum_{j=1}^m BIF(Q_{M_m^j}^t;\hat{\theta},P^M)
\ealignt
where $BIF(Q_{M_m^j}^t;\hat{\theta},P^M)$ is the influence function of individual point ${z}_{M+j}^t$ when the base model is $\hat{\theta}_M$. \\
Hence, first order approximation model $\tilde{\theta}_{M_m+1^t}$ of $\hat{\theta}_{M_+1^t}$ is given by the following:
\balignt
\tilde{\theta}_{M_m+1^t} = \hat{\theta}_M  - \frac{m}{M+m}\sum_{j=1}^m H_{\hat{\theta}_M}^{-1} ~\nabla_{\theta} \ell({z}_{M+j}^t,\hat{\theta}_M) 
\ealignt
The following also holds just by the assumptions:
\balignt
\| \tilde{\theta}_{M_m+1^t}  - \hat{\theta}_{M_m+1^t}  \| \leq \frac{B_m m^2}{(M+m)^2}
\ealignt
for some real positive constants $B_m$.
\end{proof}

\section{Analysis of Local Greedy} \label{ap:analysis}
In this section, we provide the proofs of the main results provided in the main paper. 
\begin{lemma}[Lemma~\ref{lem:exact_vs_approx}]
Let us assume that the current parameter $\theta_{M_m+k}$ is obtained by training on $M+k*m$ number of points. If $\tilde{\theta}_{M_m+k+1^g}$ and $\hat{\theta}_{M_m+k+1^g}$  are the output of Algorithm~\ref{algo:exact_greedy_m} and Algorithm~\ref{algo:approximate_greedy_m_1_rand}  for $\epsilon=0$ correspondingly after adding set of $m$ points greedily then under the bounded gradient and bounded Hessian assumption on the function of interest $R:\Omega_{\theta}\rightarrow \mathbb{R}$ , the following holds:
$$ |R(\tilde{\theta}_{M_m+k+1^g})  - R(\hat{\theta}_{M_m+k+1^g})| \in \mathcal{O}\left(\frac{m^2}{(M+(k+1)m)^2} \right)$$.
\end{lemma}
\begin{proof}
\balignt
R(\tilde{\theta}_{M_m+k+1^g})  - R(\hat{\theta}_{M_m+k+1^g}) &= R(\tilde{\theta}_{M_m+k+1^g}) - R(\tilde{\theta}_{M_m+k}) +  R(\tilde{\theta}_{M_m+k})  - R(\hat{\theta}_{M_m+k+1^g}) \\
\ealignt
Now the result directly comes from the definition of the first order approximation of the model.  From the definition, 
\balignt
R(\tilde{\theta}_{M_m+k+1^g}) - R(\tilde{\theta}_{M_m+k})  = R(\hat{\theta}_{M_m+k+1^g}) - R(\tilde{\theta}_{M_m+k})   + \mathcal{O}\left(  \frac{m^2}{(M+m*k)^2}\right).
\ealignt
Hence, the result follows. 

\end{proof}


Before directly going to the proof for arbitrary $m$ number of points selection as in of Lemma~\ref{lem:m_1_2_cons_compare} and Theorem~\ref{thm:optimality_1_point_addition}, we first prove the result for $m=1$ and the general result follows

\begin{lemma} \label{lem:1_1_2_cons_compare}
Under the assumptions discussed in the section~\ref{subsec:assumptions}, the gain in $ R $ after two greedy point selection over two random selection provided that initial given classifier is $\theta_M$ which has been trained on $M$ point, can be given as following:
\balignt
R(\tilde{\theta}_{M+{2}^g}) - \bbE[ R(\hat{\theta}_{M+{2}^r})] \leq \Delta_{\theta_{M+1^g}}' + \Delta_{\theta_{M}}'  + \frac{G'}{(M+1)^2}.
\ealignt
for some real positive constant $G'$.
\end{lemma}

\begin{proof}
From our defination in equation~\eqref{eq:delta_1_greedy_rand}, we know that
$$\Delta_{\theta_{M+k}}' = R(\theta_{M+k+1^g}) - \bbE [ R(\theta_{M+k+1^r}) ] $$
\balignt
R(\hat{\theta}_{M+{2}^g}) - \bbE[ R(\hat{\theta}_{M+{2}^r})] &=  \underbrace{R(\hat{\theta}_{M+{2}^g}) -  \bbE[R(\hat{\theta}_{M+{1}^g+1^r})]}_{:=\Delta_{\theta_{M+1^g}}'} +  \bbE[R(\hat{\theta}_{M+{1}^g+1^r})] -\bbE[ R(\hat{\theta}_{M+{2}^r})] \notag \\
&= \Delta_{\theta_{M+1^g}}' +   \bbE[R(\hat{\theta}_{M+{1}^g+1^r})] -\bbE[ R(\hat{\theta}_{M+{2}^r})] \label{eq:lem_1_p_eq1}
\ealignt
Now let us consider,
\balignt
&\bbE[R(\hat{\theta}_{M+{1}^g+1^r})] -\bbE[ R(\hat{\theta}_{M+{2}^r})] = \bbE[R(\hat{\theta}_{M+{1}^g+1^r})] -R(\hat{\theta}_{M+{1}^g}) + \overbrace{ R(\hat{\theta}_{M+{1}^g}) - \bbE[R(\hat{\theta}_{M+{1}^r})]}^{:=\Delta_{\theta_{M}}'} \notag \\
& \qquad \qquad \qquad \qquad \qquad \qquad \qquad \qquad \qquad \qquad -\left(\bbE[ R(\hat{\theta}_{M+{2}^r})] -\bbE[ R(\hat{\theta}_{M+{1}^r})] \right) \notag \\
&= \qquad \qquad \qquad  \Delta_{\theta_{M}}' + \left(\bbE[R(\hat{\theta}_{M+{1}^g+1^r})] -R(\hat{\theta}_{M+{1}^g})\right) - \left(\bbE[ R(\hat{\theta}_{M+{2}^r})] -\bbE[ R(\hat{\theta}_{M+{1}^r})] \right)  \label{eq:lem_1_p_eq2}
\ealignt
It is important to note here that, in the third term of equation~\eqref{eq:lem_1_p_eq2}, the expectaion is \textit{w.r.t} to the sampling of random points. Hence, from tower of expectation rule,  the expectation in $\bbE[ R(\hat{\theta}_{M+{2}^r})]$ can be decomposed as expectation over sampling for $1^r$ then conditional expectation of sampling for $2^r$   given $1^r$. Now, we will do the first order approximation for $2^{nd}$ and $3^{rd}$ term in equation~\eqref{eq:lem_1_p_eq2}. From arguments discussed earlier , we know that 
\balignt
\bbE\left[ R(\hat{\theta}_{M+{2}^r})|\hat{\theta}_{M+{1}^r} \right] - R(\hat{\theta}_{M+{1}^r})  = \frac{1}{M+2} \nabla_\theta R(\hat{\theta}_{M+{1}^r})^\top {H_{\hat{\theta}_{M+{1}^r}} ^{-1}} \bbE \nabla_\theta \ell(z_{M+2}^r,\hat{\theta}_{M+1^r})\notag \\
 + \frac{B_r}{(M+2)^2} \label{eq:lem_1_p_eq3}
\ealignt
for some constant $B_r$ and similarly
\balignt
\bbE[R(\hat{\theta}_{M+{1}^g+1^r})] -R(\hat{\theta}_{M+{1}^g}) \leq \frac{1}{M+2} \nabla_\theta R(\hat{\theta}_{M+{1}^g})^\top {H_{\hat{\theta}_{M+{1}^g}} ^{-1}} \bbE \nabla_\theta \ell(z_{M+2}^r,\hat{\theta}_{M+1^g}) + \frac{B_l}{(M+2)^2} \label{eq:lem_1_p_eq4}
\ealignt
Hence, from equations~\eqref{eq:lem_1_p_eq2},\eqref{eq:lem_1_p_eq3} and \eqref{eq:lem_1_p_eq4}, we get the following:
\balignt
R(\hat{\theta}_{M+{2}^g}) - \bbE[ R(\hat{\theta}_{M+{2}^r})]& \leq \Delta_{\theta_{M+1^g}}' + \Delta_{\theta_{M}}' + \frac{1}{M+2} \left[ \nabla_\theta R(\hat{\theta}_{M+{1}^g})^\top {H_{\hat{\theta}_{M+{1}^g}} ^{-1}} \bbE \nabla_\theta  \ell(z_{M+2}^r,\hat{\theta}_{M+1^g})   \right] \notag \\
&- \frac{1}{M+2}\bbE \left[ \nabla_\theta R(\hat{\theta}_{M+{1}^r})^\top {H_{\hat{\theta}_{M+{1}^r}} ^{-1}}   \bbE \nabla_\theta  \ell(z_{M+2}^r,\hat{\theta}_{M+1^r})   \right] + \frac{2B_l}{(M+2)^2} \label{eq:lem_1_p_eq5}
\ealignt
Again if,
\balignt
\overline{\nabla}_\theta R(\hat{\theta}_{M+{1}^t}) = \nabla_\theta R(\hat{\theta}_{M}) - \frac{1}{M+1}  \nabla_{\theta}^2 R(\theta_M)^\top H_{\hat{\theta}_M}^{-1} \nabla_\theta  \ell(z_{M+1}^t, \hat{\theta}_M). \notag 
\ealignt
then 
\balignt
\| \overline{\nabla}_\theta R(\hat{\theta}_{M+{1}^t}) - {\nabla}_\theta R(\hat{\theta}_{M+{1}^t}) \| = \frac{B_m}{(M+1)^2}. \label{eq:lem_1_p_eq6}
\ealignt
Similarly, we can argue about the first order approximation of the gradient of loss function. Let us denote the first order approximaion of the gradient of loss with $\overline{\nabla}_{\theta} l$ 
\balignt
\overline{\nabla}_\theta \ell(z,\hat{\theta}_{M+{1}^t}) = \nabla_\theta \ell(z,\hat{\theta}_{M}) - \frac{1}{M+1}  \nabla_{\theta}^2\ell(z,\theta_M)^\top H_{\hat{\theta}_M}^{-1} \nabla_\theta \ell(z_{M+1}^t, \hat{\theta}_M). \notag 
\ealignt
then again
\balignt
\| \overline{\nabla}_\theta\ell(z,\hat{\theta}_{M+{1}^t}) - {\nabla}_\theta\ell(z,\hat{\theta}_{M+{1}^t}) \| = \frac{B_m}{(M+1)^2}. \label{eq:lem_1_p_eq7}
\ealignt
This implies:
\balignt
 &\nabla_\theta R(\hat{\theta}_{M+{1}^t})^\top {H_{\hat{\theta}_{M+{1}^t}} ^{-1}}  \nabla_\theta \ell(z_{M+2}^r,\hat{\theta}_{M+{1}^t}) \\
 &= \left(\nabla_\theta R(\hat{\theta}_{M+{1}^t})+ \overline{\nabla}_\theta R(\hat{\theta}_{M+{1}^t})  -\overline{\nabla}_\theta R(\hat{\theta}_{M+{1}^t}) \right)^\top {H_{\hat{\theta}_{M+{1}^t}} ^{-1}}  \nabla_\theta \ell(z_{M+2}^r,\hat{\theta}_{M+{1}^t}) \\
 &= \overline{\nabla}_\theta R(\hat{\theta}_{M+{1}^t})^\top {H_{\hat{\theta}_{M+{1}^t}} ^{-1}}  \nabla_\theta \ell(z_{M+2}^r,\hat{\theta}_{M+{1}^t}) + \frac{C}{(M+1)^2} ~~ \mbox{for} ~t\in\{r,g \}
\ealignt
for some constant $C$. The last equaion comes from the equaltion~\eqref{eq:lem_1_p_eq6} and our assumptions on bounded gradient as well as bounded below singular values of the Hessian matrix. Also we would have:
\balignt
\begin{split}
 \overline{\nabla}_\theta  & R(\hat{\theta}_{M+{1}^t})^\top {H_{\hat{\theta}_{M+{1}^t}} ^{-1}}  \nabla_\theta \ell(z_{M+2}^r,\hat{\theta}_{M+{1}^t}) = {\nabla}_\theta R(\hat{\theta}_{M})^\top {H_{\hat{\theta}_{M+{1}^t}} ^{-1}}   \nabla_\theta \ell(z_{M+2}^r,\hat{\theta}_{M+{1}^t}) \\ 
 & \qquad \qquad \qquad \qquad   \qquad- \frac{1}{M+1}  \nabla_\theta \ell(z_{M+1}^t, \hat{\theta}_M)^\top H_{\hat{\theta}_M}^{-1} \nabla_{\theta}^2 R(\theta_M){H_{\hat{\theta}_{M+{1}^t}} ^{-1}}  \nabla_\theta \ell(z_{M+2}^r,\hat{\theta}_{M+{1}^t}) \\
 &= {\nabla}_\theta R(\hat{\theta}_{M})^\top {H_{\hat{\theta}_{M+{1}^t}} ^{-1}}   \left[\nabla_\theta \ell(z_{M+2}^r,\hat{\theta}_{M+{1}^t}) - \overline{\nabla}_\theta \ell(z_{M+2}^r,\hat{\theta}_{M+{1}^t}) + \overline{\nabla}_\theta \ell(z_{M+2}^r,\hat{\theta}_{M+{1}^t}) \right] \\
  & \qquad  \qquad  \qquad \qquad \qquad - \frac{1}{M+1}  \nabla_\theta \ell(z_{M+1}^t, \hat{\theta}_M)^\top H_{\hat{\theta}_M}^{-1} \nabla_{\theta}^2 R(\theta_M){H_{\hat{\theta}_{M+{1}^t}} ^{-1}}  \nabla_\theta \ell(z_{M+2}^r,\hat{\theta}_{M+{1}^t}) \\
  &= {\nabla}_\theta R(\hat{\theta}_{M})^\top {H_{\hat{\theta}_{M+{1}^t}} ^{-1}}   \left[\nabla_\theta \ell(z_{M+2}^r,\hat{\theta}_{M+{1}^t}) - \overline{\nabla}_\theta \ell(z_{M+2}^r,\hat{\theta}_{M+{1}^t}) + \overline{\nabla}_\theta \ell(z_{M+2}^r,\hat{\theta}_{M+{1}^t}) \right] \\
  &  ~~  - \frac{1}{M+1}  \nabla_\theta \ell(z_{M+1}^t, \hat{\theta}_M)^\top H_{\hat{\theta}_M}^{-1} \nabla_{\theta}^2 R(\theta_M){H_{\hat{\theta}_{M+{1}^t}} ^{-1}}  \overline{\nabla}_\theta \ell(z_{M+2}^r,\hat{\theta}_{M+{1}^t}) \\
  & ~~  - \frac{1}{M+1}  \nabla_\theta \ell(z_{M+1}^t, \hat{\theta}_M)^\top H_{\hat{\theta}_M}^{-1} \nabla_{\theta}^2 R(\theta_M){H_{\hat{\theta}_{M+{1}^t}} ^{-1}} \left[{\nabla}_\theta \ell(z_{M+2}^r,\hat{\theta}_{M+{1}^t})  - \overline{\nabla}_\theta \ell(z_{M+2}^r,\hat{\theta}_{M+{1}^t}) \right] \\
  &= {\nabla}_\theta R(\hat{\theta}_{M})^\top {H_{\hat{\theta}_{M+{1}^t}} ^{-1}}   \overline{\nabla}_\theta \ell(z_{M+2}^r,\hat{\theta}_{M+{1}^t}) \\
  & ~~ + {\nabla}_\theta R(\hat{\theta}_{M})^\top {H_{\hat{\theta}_{M+{1}^t}} ^{-1}}   \left[\nabla_\theta \ell(z_{M+2}^r,\hat{\theta}_{M+{1}^t}) - \overline{\nabla}_\theta \ell(z_{M+2}^r,\hat{\theta}_{M+{1}^t}) \right] \\
  &  ~~   - \frac{1}{M+1}  \nabla\ell(z_{M+1}^t, \hat{\theta}_M)^\top H_{\hat{\theta}_M}^{-1} \nabla_{\theta}^2 R(\theta_M){H_{\hat{\theta}_{M+{1}^t}} ^{-1}}  \overline{\nabla}_\theta \ell(z_{M+2}^r,\hat{\theta}_{M+{1}^t}) \\
  &~~ - \frac{1}{M+1}  \nabla_\theta \ell(z_{M+1}^t, \hat{\theta}_M)^\top H_{\hat{\theta}_M}^{-1} \nabla_{\theta}^2 R(\theta_M){H_{\hat{\theta}_{M+{1}^t}} ^{-1}} \left[{\nabla}_\theta \ell(z_{M+2}^r,\hat{\theta}_{M+{1}^t})  - \overline{\nabla}_\theta \ell(z_{M+2}^r,\hat{\theta}_{M+{1}^t}) \right] \\
  &= {\nabla}_\theta R(\hat{\theta}_{M})^\top {H_{\hat{\theta}_{M+{1}^t}} ^{-1}} \left[ \nabla_\theta\ell(z_{M+2}^r,\hat{\theta}_{M}) - \frac{1}{M+1}  \nabla_{\theta}^2\ell(z_{M+2}^r,\theta_M)^\top H_{\hat{\theta}_M}^{-1} \nabla_\theta \ell(z_{M+1}^t, \hat{\theta}_M) \right]   \\
  & ~~+ {\nabla}_\theta R(\hat{\theta}_{M})^\top {H_{\hat{\theta}_{M+{1}^t}} ^{-1}}   \left[\nabla_\theta \ell(z_{M+2}^r,\hat{\theta}_{M+{1}^t}) - \overline{\nabla}_\theta \ell(z_{M+2}^r,\hat{\theta}_{M+{1}^t}) \right] \\
  &   ~~ - \frac{1}{M+1}  \nabla_\theta \ell(z_{M+1}^t, \hat{\theta}_M)^\top H_{\hat{\theta}_M}^{-1} \nabla_{\theta}^2 R(\theta_M){H_{\hat{\theta}_{M+{1}^t}} ^{-1}}   \left[ \nabla_\theta\ell(z_{M+2}^r,\hat{\theta}_{M}) - \frac{1}{M+1}  \nabla_{\theta}^2\ell(z_{M+2}^r,\theta_M)^\top H_{\hat{\theta}_M}^{-1} \nabla_\theta \ell(z_{M+1}^t, \hat{\theta}_M) \right]  \\
  & ~~ - \frac{1}{M+1}  \nabla_\theta \ell(z_{M+1}^t, \hat{\theta}_M)^\top H_{\hat{\theta}_M}^{-1} \nabla_{\theta}^2 R(\theta_M){H_{\hat{\theta}_{M+{1}^t}} ^{-1}} \left[{\nabla}_\theta \ell(z_{M+2}^r,\hat{\theta}_{M+{1}^t})  - \overline{\nabla}_\theta \ell(z_{M+2}^r,\hat{\theta}_{M+{1}^t}) \right] \\
 &= {\nabla}_\theta R(\hat{\theta}_{M})^\top {H_{\hat{\theta}_{M+{1}^t}} ^{-1}} \nabla_\theta\ell(z_{M+2}^r,\hat{\theta}_{M}) - \frac{1}{M+1}  \nabla_\theta \ell(z_{M+1}^t, \hat{\theta}_M)^\top H_{\hat{\theta}_M}^{-1} \nabla_{\theta}^2 R(\theta_M){H_{\hat{\theta}_{M+{1}^t}} ^{-1}}  \nabla_\theta\ell(z_{M+2}^r,\hat{\theta}_{M}) \\
&~~ - \frac{1}{M+1}  {\nabla}_\theta R(\hat{\theta}_{M})^\top {H_{\hat{\theta}_{M+{1}^t}} ^{-1}} \nabla_{\theta}^2\ell(z_{M+2}^r,\theta_M)^\top H_{\hat{\theta}_M}^{-1} \nabla_\theta \ell(z_{M+1}^t, \hat{\theta}_M)   \\
  &   ~~ + \frac{1}{(M+1)^2}  \nabla_\theta \ell(z_{M+1}^t, \hat{\theta}_M)^\top H_{\hat{\theta}_M}^{-1} \nabla_{\theta}^2 R(\theta_M){H_{\hat{\theta}_{M+{1}^t}} ^{-1}}     \nabla_{\theta}^2\ell(z_{M+2}^r,\theta_M)^\top H_{\hat{\theta}_M}^{-1} \nabla_\theta \ell(z_{M+1}^t, \hat{\theta}_M)   \\
    &~~ + {\nabla}_\theta R(\hat{\theta}_{M})^\top {H_{\hat{\theta}_{M+{1}^t}} ^{-1}}   \left[\nabla_\theta \ell(z_{M+2}^r,\hat{\theta}_{M+{1}^t}) - \overline{\nabla}_\theta \ell(z_{M+2}^r,\hat{\theta}_{M+{1}^t}) \right] \\
  &~~ - \frac{1}{M+1}  \nabla_\theta \ell(z_{M+1}^t, \hat{\theta}_M)^\top H_{\hat{\theta}_M}^{-1} \nabla_{\theta}^2 R(\theta_M){H_{\hat{\theta}_{M+{1}^t}} ^{-1}} \left[{\nabla}_\theta \ell(z_{M+2}^r,\hat{\theta}_{M+{1}^t})  - \overline{\nabla}_\theta \ell(z_{M+2}^r,\hat{\theta}_{M+{1}^t}) \right] \label{eq:lem_1_p_eq8}
  \end{split} 
\ealignt
In equation~\ref{eq:lem_1_p_eq8}, we observe that the term $\| {\nabla}_\theta \ell(z_{M+2}^r,\hat{\theta}_{M+{1}^t})  - \overline{\nabla}_\theta \ell(z_{M+2}^r,\hat{\theta}_{M+{1}^t}) \| \in \mathcal{O} \left( \frac{1}{(M+2)^2} \right)$. Hence from equation~\eqref{eq:lem_m_p_eq8} and the observation, we have the following:

\balignt
&\overline{\nabla}_\theta  R(\hat{\theta}_{M+{1}^t})^\top {H_{\hat{\theta}_{M+{1}^t}} ^{-1}}  \nabla_\theta \ell(z_{M+2}^r,\hat{\theta}_{M+{1}^t})  = {\nabla}_\theta R(\hat{\theta}_{M})^\top {H_{\hat{\theta}_{M+{1}^t}} ^{-1}} \nabla_\theta\ell(z_{M+2}^r,\hat{\theta}_{M})\notag \\
& - \frac{1}{M+1}  \nabla_\theta \ell(z_{M+1}^t, \hat{\theta}_M)^\top H_{\hat{\theta}_M}^{-1} \nabla_{\theta}^2 R(\theta_M){H_{\hat{\theta}_{M+{1}^t}} ^{-1}}  \nabla_\theta\ell(z_{M+2}^r,\hat{\theta}_{M}) + \mathcal{O}\left( \frac{1}{(M+1)^2} \right) + \mathcal{O}\left( \frac{1}{(M+1)(M+2)^2} \right)\notag \\
& - \frac{1}{M+1}  {\nabla}_\theta R(\hat{\theta}_{M})^\top {H_{\hat{\theta}_{M+{1}^t}} ^{-1}} \nabla_{\theta}^2\ell(z_{M+2}^r,\theta_M)^\top H_{\hat{\theta}_M}^{-1} \nabla_\theta \ell(z_{M+1}^t, \hat{\theta}_M)     +  \mathcal{O}\left( \frac{1}{(M+2)^2} \right) \label{eq:lem_1_p_eq9}
\ealignt

Now, from equations~\eqref{eq:lem_1_p_eq4} and \eqref{eq:lem_1_p_eq9}, we have the follwoing:
\balignt
&R(\hat{\theta}_{M+{2}^g}) - \bbE[ R(\hat{\theta}_{M+{2}^r})] \leq \Delta_{\theta_{M+1^g}}' + \Delta_{\theta_{M}}' + \frac{1}{M+2} \left[ \nabla_\theta R(\hat{\theta}_{M+{1}^g})^\top {H_{\hat{\theta}_{M+{1}^g}} ^{-1}} \bbE \nabla_\theta \ell(z_{M+2}^r,\hat{\theta}_{M+1^g})   \right] \notag \\
& \qquad  \qquad  \qquad  \qquad  \qquad  \qquad  \qquad - \frac{1}{M+2}\bbE \left[ \nabla_\theta R(\hat{\theta}_{M+{1}^r})^\top {H_{\hat{\theta}_{M+{1}^r}} ^{-1}}   \bbE \nabla_\theta \ell(z_{M+2}^r,\hat{\theta}_{M+1^r})   \right] + \frac{2B_l}{(M+2)^2} \notag \\
&= \Delta_{\theta_{M+1^g}}' + \Delta_{\theta_{M}}' + \frac{1}{M+2} \bbE\left[  {\nabla}_\theta R(\hat{\theta}_{M})^\top \left({H_{\hat{\theta}_{M+{1}^g}} ^{-1}} - {H_{\hat{\theta}_{M+{1}^r}} ^{-1}} \right) \nabla_\theta\ell(z_{M+2}^r,\hat{\theta}_{M}) \right] \notag \\
& - \frac{1}{(M+1)(M+2)} \Big[ [\nabla_\theta  \ell(z_{M+1}^g, \hat{\theta}_M) - \nabla_\theta  \ell(z_{M+1}^r, \hat{\theta}_M) ]^\top H_{\hat{\theta}_M}^{-1} \nabla_{\theta}^2 R(\theta_M){H_{\hat{\theta}_{M+{1}^g}} ^{-1}}  \nabla_\theta \ell(z_{M+2}^r,\hat{\theta}_{M}) \Big] \notag \\
& + \frac{1}{(M+1)(M+2)} \Big[ \nabla_\theta  \ell(z_{M+1}^r, \hat{\theta}_M) ^\top H_{\hat{\theta}_M}^{-1} \nabla_{\theta}^2 R(\theta_M)({H_{\hat{\theta}_{M+{1}^g}} ^{-1}} - {H_{\hat{\theta}_{M+{1}^r}} ^{-1}} )  \nabla_\theta \ell(z_{M+2}^r,\hat{\theta}_{M}) \Big] \notag \\
& -  \frac{1}{(M+1)(M+2)}  {\nabla}_\theta R(\hat{\theta}_{M})^\top ( {H_{\hat{\theta}_{M+{1}^g}} ^{-1}} - {H_{\hat{\theta}_{M+{1}^g}} ^{-1}} )\nabla_{\theta}^2\ell(z_{M+2}^r,\theta_M)^\top H_{\hat{\theta}_M}^{-1} \nabla_\theta \ell(z_{M+1}^g, \hat{\theta}_M)  \notag \\
& +  \frac{1}{(M+1)(M+2)}  {\nabla}_\theta R(\hat{\theta}_{M})^\top  {H_{\hat{\theta}_{M+{1}^g}} ^{-1}} \nabla_{\theta}^2\ell(z_{M+2}^r,\theta_M)^\top H_{\hat{\theta}_M}^{-1}(\nabla_\theta  \ell(z_{M+1}^g, \hat{\theta}_M) -  \nabla_\theta  \ell(z_{M+1}^r, \hat{\theta}_M))+  \mathcal{O}\left( \frac{1}{M^3} \right) \label{eq:lem_1_p_eq10}
 \ealignt
We do have:
\balignt
{H_{\hat{\theta}_{M+{1}^g}} ^{-1}} - {H_{\hat{\theta}_{M+{1}^r}} ^{-1}}  = {H_{\hat{\theta}_{M+{1}^g}} ^{-1}}  (H_{\hat{\theta}_{M+{1}^r}} - H_{\hat{\theta}_{M+{1}^g}}) {H_{\hat{\theta}_{M+{1}^r}} ^{-1}}
\ealignt
If we assume that the Hessian inverse operator norm is bounded by $H$ and we have $\hat{\theta}_{M+{1}^g} - \hat{\theta}_{M+{1}^r} = \frac{1}{M+1} H_{\hat{\theta}_M}^{-1} \big(\nabla_{\theta} \ell (z^{r}_{M+1}) - \nabla_{\theta} \ell (z^{g}_{M+1}) \big)$. Now from the Lipschitz Hessian assumption:
\balignt
\|{H_{\hat{\theta}_{M+{1}^g}} ^{-1}} - {H_{\hat{\theta}_{M+{1}^r}} ^{-1}}\| = {H_{\hat{\theta}_{M+{1}^g}} ^{-1}} ({H_{\hat{\theta}_{M+{1}^r}} }  - {H_{\hat{\theta}_{M+{1}^g}} } ) {H_{\hat{\theta}_{M+{1}^r}} ^{-1}}  \leq L H^2 \|\hat{\theta}_{M+{1}^g} - \hat{\theta}_{M+{1}^r}   \| \leq \frac{G}{M+1} \label{eq:lem_1_p_eq11}
\ealignt
for some constant $G$.
Now from the equation~\eqref{eq:lem_1_p_eq10} and \eqref{eq:lem_1_p_eq11}, we have the following:
\balignt
R(\hat{\theta}_{M+{2}^g}) - \bbE[ R(\hat{\theta}_{M+{2}^r})] \leq \Delta_{\theta_{M+1^g}}' + \Delta_{\theta_{M}}'  + \frac{G}{(M+1)(M+2)} + \mathcal{O}\left(\frac{1}{(M+1)(M+2)^2}\right)
\ealignt
Last equation comes from equation~\eqref{eq:lem_1_p_eq11}  and from the fact that $\| \nabla_\theta  \ell(z_{M+1}^g, \hat{\theta}_M) - \nabla_\theta  \ell(z_{M+1}^r, \hat{\theta}_M) \| \in \mathcal{O}\left( \frac{1}{M+2} \right)$. Hence for some constant $G'$:
\balignt
R(\hat{\theta}_{M+{2}^g}) - \bbE[ R(\hat{\theta}_{M+{2}^r})] \leq \Delta_{\theta_{M+1^g}}' + \Delta_{\theta_{M}}'  + \frac{G'}{(M+1)(M+2)}.
\ealignt
From lemma~\ref{lem:exact_vs_approx}, we know that $|R(\hat{\theta}_{M+{2}^g}) - R(\tilde{\theta}_{M+{2}^g}) | \in \mathcal{O}\left(  \frac{1}{(M+2)^2}\right)$. Similarly $|R(\hat{\theta}_{M+{1}^g}) - R(\tilde{\theta}_{M+{1}^g}) | \in \mathcal{O}\left(  \frac{1}{(M+1)^2}\right)$.
Hence, we get our final result:
\balignt
R(\tilde{\theta}_{M+{2}^g}) - \bbE[ R(\hat{\theta}_{M+{2}^r})]\leq \Delta_{\theta_{M+1^g}}' + \Delta_{\theta_{M}}'  + \frac{G'}{(M+1)(M+2)} \leq  \Delta_{\theta_{M+1^g}}' + \Delta_{\theta_{M}}'  + \frac{G'}{(M+1)^2}
\ealignt
\end{proof}

\begin{theorem} \label{thm:optimality_1_point_addition}
 If we run our greedy algorithm for $p$ successive steps and all the assumptions discussed in the section~\ref{subsec:assumptions} are satisfied then the gain in $ R $ after $p$ successive greedy steps  over the same number of random selection starting from the classifier $\theta_M$ can be characterized as follows:
 \balignt
 R(\tilde{\theta}_{M+{p}^g}) -  \bbE[ R(\hat{\theta}_{M+{p}^r})] \leq \sum_{i = 0}^{p-1}\Delta_{\theta_{M+i^g}}'  + \sum_{i = 1}^{p-1}\frac{\tilde{G}}{(M+i)^2},
 \ealignt
 for some positive constant $\tilde{G}$ where $\tilde{\theta}_{M+{p}^g}$ is the output from the approximate algorithm~\ref{algo:approximate_greedy_m_1_rand} for $\epsilon=0$.
 \end{theorem}
\begin{proof}
Let us assume that $p$ is a multiple of $2$. The proof for more number of iterations can also be seen as just unrolling proof being done in the lemma~\ref{lem:m_1_2_cons_compare}.
\balignt
\begin{split}\label{eq:thm_1_p_eq1}
R(\hat{\theta}_{M+{p}^g}) - \bbE[ R(\hat{\theta}_{M+{p}^r})] &=  \underbrace{R(\hat{\theta}_{M+{p}^g}) -  \bbE[R(\hat{\theta}_{M+{(p-1)}^g+1^r})]}_{:=\Delta_{\theta_{M+(p-1)^g}}'} +  \bbE[R(\hat{\theta}_{M+{(p-1)}^g+1^r})] -\bbE[ R(\hat{\theta}_{M+{p}^r})]  \\
&= \Delta_{\theta_{M+(p-1)^g}}' +   \bbE[R(\hat{\theta}_{M+{(p-1)}^g+1^r})] -\bbE[ R(\hat{\theta}_{M+{p}^r})]   \\
&= \Delta_{\theta_{M+(p-1)^g}}' +  \bbE[R(\hat{\theta}_{M+{(p-1)}^g+1^r})] -  R(\hat{\theta}_{M+{(p-1)}^g}) + R(\hat{\theta}_{M+{(p-1)}^g})  \\
& \qquad - \bbE[ R(\hat{\theta}_{M+{(p-1)}^r})] + \bbE[ R(\hat{\theta}_{M+{(p-1)}^r})]  -\bbE[ R(\hat{\theta}_{M+{p}^r})] \\
&= \Delta_{\theta_{M+(p-1)^g}}' +  \bbE[R(\hat{\theta}_{M+{(p-1)}^g+1^r})] -  R(\hat{\theta}_{M+{(p-1)}^g}) \\
&  - \left(\bbE[ R(\hat{\theta}_{M+{p}^r})] - \bbE[ R(\hat{\theta}_{M+{(p-1)}^r})]\right)+ \underbrace{R(\hat{\theta}_{M+{(p-1)}^g}) - \bbE[ R(\hat{\theta}_{M+{(p-1)}^r})]}_{:=\text{recursion term}}
\end{split}
\ealignt

The bound on the term $\bbE[R(\hat{\theta}_{M+{(p-1)}^g+1^r})] -  R(\hat{\theta}_{M+{(p-1)}^g}) - \left(\bbE[ R(\hat{\theta}_{M+{p}^r})] - \bbE[ R(\hat{\theta}_{M+{(p-1)}^r})]\right)$ has been obtained in the proof of lemma~\ref{lem:m_1_2_cons_compare} which says:
\balignt
\bbE[R(\hat{\theta}_{M+{(p-1)}^g+1^r})] -  R(\hat{\theta}_{M+{(p-1)}^g}) - \left(\bbE[ R(\hat{\theta}_{M+{p}^r})] - \bbE[ R(\hat{\theta}_{M+{(p-1)}^r})]\right) \leq \frac{G'}{(M+p-1)(p+1)}
\ealignt
for some real positive constant $G'$. Similarly, we can get the bound on the recursion term which essentially gives  the following result:
\balignt
R(\hat{\theta}_{M+{p}^g}) -  \bbE[ R(\hat{\theta}_{M+{p}^r})] \leq \sum_{i = 0}^{p-1}\Delta_{\theta_{M+i^g}}'  + \sum_{i = 1}^{p-1}\frac{{G}'}{(M+i)^2}.
\ealignt
We finally apply lemma~\ref{lem:exact_vs_approx} to get our final result which is 
\balignt
R(\tilde{\theta}_{M+{p}^g}) -  \bbE[ R(\hat{\theta}_{M+{p}^r})] \leq \sum_{i = 0}^{p-1}\Delta_{\theta_{M+i^g}}'  + \sum_{i = 1}^{p-1}\frac{\tilde{G}}{(M+i)^2}
\ealignt
for some positive real constant $\tilde{G}$.

\end{proof}

Proof for general $m$ follows exactly the same path but we provide the proof with theorem statement here for the completeness. 
We now state Lemma~\ref{lem:m_1_2_cons_compare} in which we prove that difference in the objective value $R$ after two rounds of greedy subset selection vs two rounds of random subset selection can be bounded under the Assumptions made in this paper. 
\paragraph{Notations for the Proof of Lemma~\ref{lem:m_1_2_cons_compare} and Theorem~\ref{thm:optimality_m_point_addition}-} In Lemma~\ref{lem:m_1_2_cons_compare} and Theorem~\ref{thm:optimality_m_point_addition}, we add $m$ points at a time. Everywhere where we have $\nabla_\theta \ell(z,\hat{\theta})$ it denotes the gradient is sum over gradient of  $m$ points which are there in the batch $z$ \textit{i.e. } $\nabla_\theta \ell(z,\hat{\theta}) = \sum_{z_i \in z} \nabla_\theta \ell(z_i,\hat{\theta}) $.

\begin{lemma}[Lemma~\ref{lem:m_1_2_cons_compare}] 
Under the assumptions discussed in the section~\ref{subsec:assumptions}, the gain in $ R $ after two rounds of greedy point selection over two rounds of random selection provided that initial given classifier is $\theta_M$ which has been trained on $M$ point, can be given as following:
\balignt
R(\tilde{\theta}_{M_m+{2}^g}) - \bbE[ R(\hat{\theta}_{M_m+{2}^r})] \leq \Delta_{\theta_{M_m+1^g}}' + \Delta_{\theta_{M}}'  + \frac{G'm^2}{(M+m)^2},
\ealignt
for some real positive constant $G'$.
\end{lemma}

\begin{proof}

From our defination in equation~\eqref{eq:delta_1_greedy_rand}, we know that
\small
$$\Delta_{\theta_{M_m+k}}' = R(\theta_{M_m+k+1^g}) - \bbE [ R(\theta_{M_m+k+1^r}) ] $$
\balignt
R(\hat{\theta}_{M_m+{2}^g}) - \bbE[ R(\hat{\theta}_{M_m+{2}^r})] &=  \underbrace{R(\hat{\theta}_{M_m+{2}^g}) -  \bbE[R(\hat{\theta}_{M_m+{1}^g+1^r})]}_{:=\Delta_{\theta_{M_m+1^g}}'} +  \bbE[R(\hat{\theta}_{M_m+{1}^g+1^r})] -\bbE[ R(\hat{\theta}_{M_m+{2}^r})] \notag \\
&= \Delta_{\theta_{M_m+1^g}}' +   \bbE[R(\hat{\theta}_{M_m+{1}^g+1^r})] -\bbE[ R(\hat{\theta}_{M_m+{2}^r})] \label{eq:lem_1_p_eq1}
\ealignt
\normalsize
Now let us consider,
\small
\balignt
&\bbE[R(\hat{\theta}_{M_m+{1}^g+1^r})] -\bbE[ R(\hat{\theta}_{M_m+{2}^r})] = \bbE[R(\hat{\theta}_{M_m+{1}^g+1^r})] -R(\hat{\theta}_{M_m+{1}^g}) + \overbrace{ R(\hat{\theta}_{M_m+{1}^g}) - \bbE[R(\hat{\theta}_{M_m+{1}^r})]}^{:=\Delta_{\theta_{M_m}}'} \notag \\
& \qquad \qquad \qquad \qquad \qquad \qquad \qquad \qquad \qquad \qquad -\left(\bbE[ R(\hat{\theta}_{M_m+{2}^r})] -\bbE[ R(\hat{\theta}_{M_m+{1}^r})] \right) \notag \\
&= \qquad \qquad \qquad  \Delta_{\theta_{M_m}}' + \left(\bbE[R(\hat{\theta}_{M_m+{1}^g+1^r})] -R(\hat{\theta}_{M_m+{1}^g})\right) - \left(\bbE[ R(\hat{\theta}_{M_m+{2}^r})] -\bbE[ R(\hat{\theta}_{M_m+{1}^r})] \right)  \label{eq:lem_m_p_eq2}
\ealignt
\normalsize
It is important to note here that, in the third term of equation~\eqref{eq:lem_1_p_eq2}, the expectaion is \textit{w.r.t} to the sampling of random points. Hence, from tower of expectation rule,  the expectation in $\bbE[ R(\hat{\theta}_{M+{2}^r})]$ can be decomposed as expectation over batch sampling for $1_m^r$ then conditional expectation of sampling for $2_m^r$   given $1_m^r$. Now, we will do the first order approximation for $2^{nd}$ and $3^{rd}$ term in equation~\eqref{eq:lem_1_p_eq2}. From arguments discussed earlier , we know that 
\small
\balignt
\bbE\left[ R(\hat{\theta}_{M_m+{2}^r})|\hat{\theta}_{M_m+{1}^r} \right] - R(\hat{\theta}_{M_m+{1}^r})  = \frac{m}{M+2m} \nabla_\theta R(\hat{\theta}_{M_m+{1}^r})^\top {H_{\hat{\theta}_{M_m+{1}^r}} ^{-1}} \bbE \nabla_\theta \ell(z_{M+2}^r,\hat{\theta}_{M_m+1^r})\notag \\
 + \frac{B_r m^2}{(M+2m)^2} \label{eq:lem_m_p_eq3}
\ealignt
\normalsize
for some constant $B_r$ and similarly
\balignt
\bbE[R(\hat{\theta}_{M_m+{1}^g+1^r})] -R(\hat{\theta}_{M_m+{1}^g}) \leq \frac{m}{M+2m} \nabla_\theta R(\hat{\theta}_{M_m+{1}^g})^\top {H_{\hat{\theta}_{M_m+{1}^g}} ^{-1}} \bbE \nabla_\theta \ell(z_{M+2}^r,\hat{\theta}_{M_m+1^g}) + \frac{B_lm^2}{(M+2m)^2} \label{eq:lem_m_p_eq4}
\ealignt
Hence, from equations~\eqref{eq:lem_m_p_eq2},\eqref{eq:lem_m_p_eq3} and \eqref{eq:lem_m_p_eq4}, we get the following:
\balignt
R(\hat{\theta}_{M_m+{2}^g}) - \bbE[ R(\hat{\theta}_{M_m+{2}^r})]& \leq \Delta_{\theta_{M_m+1^g}}' + \Delta_{\theta_{M}}' + \frac{m}{M+2m} \left[ \nabla_\theta R(\hat{\theta}_{M_m+{1}^g})^\top {H_{\hat{\theta}_{M_m+{1}^g}} ^{-1}} \bbE \nabla_\theta  \ell(z_{M+2}^r,\hat{\theta}_{M_m+1^g})   \right] \notag \\
&- \frac{m}{M+2m}\bbE \left[ \nabla_\theta R(\hat{\theta}_{M_m+{1}^r})^\top {H_{\hat{\theta}_{M_m+{1}^r}} ^{-1}}   \bbE \nabla_\theta  \ell(z_{M+2}^r,\hat{\theta}_{M_m+1^r})   \right] + \frac{2B_l m}{(M+2m)^2} \label{eq:lem_m_p_eq5}
\ealignt
Again if,
\balignt
\overline{\nabla}_\theta R(\hat{\theta}_{M_m+{1}^t}) = \nabla_\theta R(\hat{\theta}_{M}) - \frac{m}{M+m}  \nabla_{\theta}^2 R(\theta_M)^\top H_{\hat{\theta}_M}^{-1} \nabla_\theta  \ell(z_{M+1}^t, \hat{\theta}_M). \notag 
\ealignt
then 
\balignt
\| \overline{\nabla}_\theta R(\hat{\theta}_{M_m+{1}^t}) - {\nabla}_\theta R(\hat{\theta}_{M_m+{1}^t}) \| = \frac{B_m m^2}{(M+m)^2}. \label{eq:lem_m_p_eq6}
\ealignt
Similarly, we can argue about the first order approximation of the gradient of loss function. Let us denote the first order approximaion of the gradient of loss with $\overline{\nabla}_{\theta} l$ 
\balignt
\overline{\nabla}_\theta \ell(z,\hat{\theta}_{M_m+{1}^t}) = \nabla_\theta \ell(z,\hat{\theta}_{M}) - \frac{m}{M+m}  \nabla_{\theta}^2\ell(z,\theta_M)^\top H_{\hat{\theta}_M}^{-1} \nabla_\theta \ell(z_{M+1}^t, \hat{\theta}_M). \notag 
\ealignt
then again
\balignt
\| \overline{\nabla}_\theta\ell(z,\hat{\theta}_{M_m+{1}^t}) - {\nabla}_\theta\ell(z,\hat{\theta}_{M_m+{1}^t}) \| = \frac{B_mm^2 }{(M+m)^2}. \label{eq:lem_m_p_eq7}
\ealignt
This implies:
\balignt
 &\nabla_\theta R(\hat{\theta}_{M_m+{1}^t})^\top {H_{\hat{\theta}_{M_m+{1}^t}} ^{-1}}  \nabla_\theta \ell(z_{M+2}^r,\hat{\theta}_{M_m+{1}^t}) \\
 &= \left(\nabla_\theta R(\hat{\theta}_{M_m+{1}^t})+ \overline{\nabla}_\theta R(\hat{\theta}_{M_m+{1}^t})  -\overline{\nabla}_\theta R(\hat{\theta}_{M_m+{1}^t}) \right)^\top {H_{\hat{\theta}_{M_m+{1}^t}} ^{-1}}  \nabla_\theta \ell(z_{M+2}^r,\hat{\theta}_{M_m+{1}^t}) \\
 &= \overline{\nabla}_\theta R(\hat{\theta}_{M_m+{1}^t})^\top {H_{\hat{\theta}_{M_m+{1}^t}} ^{-1}}  \nabla_\theta \ell(z_{M+2}^r,\hat{\theta}_{M_m+{1}^t}) + \frac{C m^2}{(M+m)^2} ~~ \mbox{for} ~t\in\{r,g \}
\ealignt
for some constant $C$. The last equaion comes from the equaltion~\eqref{eq:lem_m_p_eq6} and our assumptions on bounded gradient as well as bounded below singular values of the Hessian matrix. Also we would have:
\small
\balignt
\begin{split}
 \overline{\nabla}_\theta  & R(\hat{\theta}_{M_m+{1}^t})^\top {H_{\hat{\theta}_{M_m+{1}^t}} ^{-1}}  \nabla_\theta \ell(z_{M+2}^r,\hat{\theta}_{M_m+{1}^t}) = {\nabla}_\theta R(\hat{\theta}_{M})^\top {H_{\hat{\theta}_{M_m+{1}^t}} ^{-1}}   \nabla_\theta \ell(z_{M+2}^r,\hat{\theta}_{M_m+{1}^t}) \\ 
 & \qquad \qquad \qquad \qquad   \qquad- \frac{m}{M+m}  \nabla_\theta \ell(z_{M+1}^t, \hat{\theta}_M)^\top H_{\hat{\theta}_M}^{-1} \nabla_{\theta}^2 R(\theta_M){H_{\hat{\theta}_{M_m+{1}^t}} ^{-1}}  \nabla_\theta \ell(z_{M+2}^r,\hat{\theta}_{M_m+{1}^t}) \\
 &= {\nabla}_\theta R(\hat{\theta}_{M})^\top {H_{\hat{\theta}_{M_m+{1}^t}} ^{-1}}   \left[\nabla_\theta \ell(z_{M+2}^r,\hat{\theta}_{M_m+{1}^t}) - \overline{\nabla}_\theta \ell(z_{M+2}^r,\hat{\theta}_{M_m+{1}^t}) + \overline{\nabla}_\theta \ell(z_{M+2}^r,\hat{\theta}_{M_m+{1}^t}) \right] \\
  & \qquad  \qquad  \qquad \qquad \qquad - \frac{m}{M+m}  \nabla_\theta \ell(z_{M+1}^t, \hat{\theta}_M)^\top H_{\hat{\theta}_M}^{-1} \nabla_{\theta}^2 R(\theta_M){H_{\hat{\theta}_{M_m+{1}^t}} ^{-1}}  \nabla_\theta \ell(z_{M+2}^r,\hat{\theta}_{M_m+{1}^t}) \\
  &= {\nabla}_\theta R(\hat{\theta}_{M})^\top {H_{\hat{\theta}_{M_m+{1}^t}} ^{-1}}   \left[\nabla_\theta \ell(z_{M+2}^r,\hat{\theta}_{M_m+{1}^t}) - \overline{\nabla}_\theta \ell(z_{M+2}^r,\hat{\theta}_{M_m+{1}^t}) + \overline{\nabla}_\theta \ell(z_{M+2}^r,\hat{\theta}_{M_m+{1}^t}) \right] \\
  &  ~~  - \frac{m}{M+m}  \nabla_\theta \ell(z_{M+1}^t, \hat{\theta}_M)^\top H_{\hat{\theta}_M}^{-1} \nabla_{\theta}^2 R(\theta_M){H_{\hat{\theta}_{M_m+{1}^t}} ^{-1}}  \overline{\nabla}_\theta \ell(z_{M+2}^r,\hat{\theta}_{M_m+{1}^t}) \\
  & ~~  - \frac{m}{M+m}  \nabla_\theta \ell(z_{M+1}^t, \hat{\theta}_M)^\top H_{\hat{\theta}_M}^{-1} \nabla_{\theta}^2 R(\theta_M){H_{\hat{\theta}_{M_m+{1}^t}} ^{-1}} \left[{\nabla}_\theta \ell(z_{M+2}^r,\hat{\theta}_{M_m+{1}^t})  - \overline{\nabla}_\theta \ell(z_{M+2}^r,\hat{\theta}_{M_m+{1}^t}) \right] \\
  &= {\nabla}_\theta R(\hat{\theta}_{M})^\top {H_{\hat{\theta}_{M_m+{1}^t}} ^{-1}}   \overline{\nabla}_\theta \ell(z_{M+2}^r,\hat{\theta}_{M_m+{1}^t}) \\
  & ~~ + {\nabla}_\theta R(\hat{\theta}_{M})^\top {H_{\hat{\theta}_{M_m+{1}^t}} ^{-1}}   \left[\nabla_\theta \ell(z_{M+2}^r,\hat{\theta}_{M_m+{1}^t}) - \overline{\nabla}_\theta \ell(z_{M+2}^r,\hat{\theta}_{M_m+{1}^t}) \right] \\
  &  ~~   - \frac{m}{M+m}  \nabla\ell(z_{M+1}^t, \hat{\theta}_M)^\top H_{\hat{\theta}_M}^{-1} \nabla_{\theta}^2 R(\theta_M){H_{\hat{\theta}_{M_m+{1}^t}} ^{-1}}  \overline{\nabla}_\theta \ell(z_{M_m+2}^r,\hat{\theta}_{M_m+{1}^t}) \\
  &~~ - \frac{m}{M+m}  \nabla_\theta \ell(z_{M+1}^t, \hat{\theta}_M)^\top H_{\hat{\theta}_M}^{-1} \nabla_{\theta}^2 R(\theta_M){H_{\hat{\theta}_{M_m+{1}^t}} ^{-1}} \left[{\nabla}_\theta \ell(z_{M+2}^r,\hat{\theta}_{M_m+{1}^t})  - \overline{\nabla}_\theta \ell(z_{M+2}^r,\hat{\theta}_{M_m+{1}^t}) \right] \\
  &= {\nabla}_\theta R(\hat{\theta}_{M})^\top {H_{\hat{\theta}_{M_m+{1}^t}} ^{-1}} \left[ \nabla_\theta\ell(z_{M+2}^r,\hat{\theta}_{M}) - \frac{m}{M+m}  \nabla_{\theta}^2\ell(z_{M+2}^r,\theta_M)^\top H_{\hat{\theta}_M}^{-1} \nabla_\theta \ell(z_{M+1}^t, \hat{\theta}_M) \right]   \\
  & ~~+ {\nabla}_\theta R(\hat{\theta}_{M})^\top {H_{\hat{\theta}_{M_m+{1}^t}} ^{-1}}   \left[\nabla_\theta \ell(z_{M+2}^r,\hat{\theta}_{M_m+{1}^t}) - \overline{\nabla}_\theta \ell(z_{M+2}^r,\hat{\theta}_{M_m+{1}^t}) \right] \\
  &   ~~ - \frac{m}{M+m}  \nabla_\theta \ell(z_{M+1}^t, \hat{\theta}_M)^\top H_{\hat{\theta}_M}^{-1} \nabla_{\theta}^2 R(\theta_M){H_{\hat{\theta}_{M_m+{1}^t}} ^{-1}}   \left[ \nabla_\theta\ell(z_{M+2}^r,\hat{\theta}_{M}) - \frac{m}{M+m}  \nabla_{\theta}^2\ell(z_{M+2}^r,\theta_M)^\top H_{\hat{\theta}_M}^{-1} \nabla_\theta \ell(z_{M+1}^t, \hat{\theta}_M) \right]  \\
  & ~~ - \frac{m}{M+m}  \nabla_\theta \ell(z_{M+1}^t, \hat{\theta}_M)^\top H_{\hat{\theta}_M}^{-1} \nabla_{\theta}^2 R(\theta_M){H_{\hat{\theta}_{M_m+{1}^t}} ^{-1}} \left[{\nabla}_\theta \ell(z_{M+2}^r,\hat{\theta}_{M_m+{1}^t})  - \overline{\nabla}_\theta \ell(z_{M+2}^r,\hat{\theta}_{M_m+{1}^t}) \right] \\
 &= {\nabla}_\theta R(\hat{\theta}_{M})^\top {H_{\hat{\theta}_{M_m+{1}^t}} ^{-1}} \nabla_\theta\ell(z_{M+2}^r,\hat{\theta}_{M}) - \frac{m}{M+m}  \nabla_\theta \ell(z_{M+1}^t, \hat{\theta}_M)^\top H_{\hat{\theta}_M}^{-1} \nabla_{\theta}^2 R(\theta_M){H_{\hat{\theta}_{M_m+{1}^t}} ^{-1}}  \nabla_\theta\ell(z_{M+2}^r,\hat{\theta}_{M}) \\
&~~ - \frac{m}{M+m}  {\nabla}_\theta R(\hat{\theta}_{M})^\top {H_{\hat{\theta}_{M_m+{1}^t}} ^{-1}} \nabla_{\theta}^2\ell(z_{M+2}^r,\theta_M)^\top H_{\hat{\theta}_M}^{-1} \nabla_\theta \ell(z_{M+1}^t, \hat{\theta}_M)   \\
  &   ~~ + \frac{m^2}{(M+m)^2}  \nabla_\theta \ell(z_{M+1}^t, \hat{\theta}_M)^\top H_{\hat{\theta}_M}^{-1} \nabla_{\theta}^2 R(\theta_M){H_{\hat{\theta}_{M_m+{1}^t}} ^{-1}}     \nabla_{\theta}^2\ell(z_{M+2}^r,\theta_M)^\top H_{\hat{\theta}_M}^{-1} \nabla_\theta \ell(z_{M+1}^t, \hat{\theta}_M)   \\
    &~~ + {\nabla}_\theta R(\hat{\theta}_{M})^\top {H_{\hat{\theta}_{M_m+{1}^t}} ^{-1}}   \left[\nabla_\theta \ell(z_{M+2}^r,\hat{\theta}_{M_m+{1}^t}) - \overline{\nabla}_\theta \ell(z_{M+2}^r,\hat{\theta}_{M_m+{1}^t}) \right] \\
  &~~ - \frac{m}{M+m}  \nabla_\theta \ell(z_{M+1}^t, \hat{\theta}_M)^\top H_{\hat{\theta}_M}^{-1} \nabla_{\theta}^2 R(\theta_M){H_{\hat{\theta}_{M_m+{1}^t}} ^{-1}} \left[{\nabla}_\theta \ell(z_{M+2}^r,\hat{\theta}_{M_m+{1}^t})  - \overline{\nabla}_\theta \ell(z_{M+2}^r,\hat{\theta}_{M_m+{1}^t}) \right] \label{eq:lem_m_p_eq8}
  \end{split} 
\ealignt
\normalsize
In equation~\ref{eq:lem_m_p_eq8}, we observe that the term $\| {\nabla}_\theta \ell(z_{M+2}^r,\hat{\theta}_{M_m+{1}^t})  - \overline{\nabla}_\theta \ell(z_{M+2}^r,\hat{\theta}_{M_m+{1}^t}) \| \in \mathcal{O} \left( \frac{m^2}{(M+2m)^2} \right)$. Hence from equation~\eqref{eq:lem_m_p_eq8} and the observation, we have the following:
\small
\balignt
&\overline{\nabla}_\theta  R(\hat{\theta}_{M_m+{1}^t})^\top {H_{\hat{\theta}_{M_m+{1}^t}} ^{-1}}  \nabla_\theta \ell(z_{M+2}^r,\hat{\theta}_{M_m+{1}^t})  = {\nabla}_\theta R(\hat{\theta}_{M})^\top {H_{\hat{\theta}_{M_m+{1}^t}} ^{-1}} \nabla_\theta\ell(z_{M+2}^r,\hat{\theta}_{M})\notag \\
& - \frac{m}{M+m}  \nabla_\theta \ell(z_{M+1}^t, \hat{\theta}_M)^\top H_{\hat{\theta}_M}^{-1} \nabla_{\theta}^2 R(\theta_M){H_{\hat{\theta}_{M_m+{1}^t}} ^{-1}}  \nabla_\theta\ell(z_{M+2}^r,\hat{\theta}_{M}) + \mathcal{O}\left( \frac{m^2}{(M+m)^2} \right) + \mathcal{O}\left( \frac{m^3}{(M+m)(M+2m)^2} \right)\notag \\
& - \frac{m}{M+m}  {\nabla}_\theta R(\hat{\theta}_{M})^\top {H_{\hat{\theta}_{M_m+{1}^t}} ^{-1}} \nabla_{\theta}^2\ell(z_{M+2}^r,\theta_M)^\top H_{\hat{\theta}_M}^{-1} \nabla_\theta \ell(z_{M+1}^t, \hat{\theta}_M)     +  \mathcal{O}\left( \frac{m^2}{(M+2m)^2} \right) \label{eq:lem_m_p_eq9}
\ealignt
\normalsize
Now, from equations~\eqref{eq:lem_m_p_eq4} and \eqref{eq:lem_m_p_eq9}, we have the follwoing:
\small
\balignt
&R(\hat{\theta}_{M_m+{2}^g}) - \bbE[ R(\hat{\theta}_{M_m+{2}^r})] \leq \Delta_{\theta_{M_m+1^g}}' + \Delta_{\theta_{M}}' + \frac{m}{M+2m} \left[ \nabla_\theta R(\hat{\theta}_{M_m+{1}^g})^\top {H_{\hat{\theta}_{M_m+{1}^g}} ^{-1}} \bbE \nabla_\theta \ell(z_{M+2}^r,\hat{\theta}_{M_m+1^g})   \right] \notag \\
& \qquad  \qquad  \qquad  \qquad  \qquad  \qquad  \qquad - \frac{m}{M+2m}\bbE \left[ \nabla_\theta R(\hat{\theta}_{M_m+{1}^r})^\top {H_{\hat{\theta}_{M_m+{1}^r}} ^{-1}}   \bbE \nabla_\theta \ell(z_{M+2}^r,\hat{\theta}_{M_m+1^r})   \right] + \frac{2B_l m^2}{(M+2m)^2} \notag \\
&= \Delta_{\theta_{M_m+1^g}}' + \Delta_{\theta_{M}}' + \frac{m}{M+2m} \bbE\left[  {\nabla}_\theta R(\hat{\theta}_{M})^\top \left({H_{\hat{\theta}_{M_m+{1}^g}} ^{-1}} - {H_{\hat{\theta}_{M_m+{1}^r}} ^{-1}} \right) \nabla_\theta\ell(z_{M+2}^r,\hat{\theta}_{M}) \right] \notag \\
& - \frac{m^2}{(M+m)(M+2m)} \Big[ [\nabla_\theta  \ell(z_{M+1}^g, \hat{\theta}_M) - \nabla_\theta  \ell(z_{M+1}^r, \hat{\theta}_M) ]^\top H_{\hat{\theta}_M}^{-1} \nabla_{\theta}^2 R(\theta_M){H_{\hat{\theta}_{M_m+{1}^g}} ^{-1}}  \nabla_\theta \ell(z_{M+2}^r,\hat{\theta}_{M}) \Big] \notag \\
& + \frac{m^2}{(M+m)(M+2m)} \Big[ \nabla_\theta  \ell(z_{M+1}^r, \hat{\theta}_M) ^\top H_{\hat{\theta}_M}^{-1} \nabla_{\theta}^2 R(\theta_M)({H_{\hat{\theta}_{M_m+{1}^g}} ^{-1}} - {H_{\hat{\theta}_{M_m+{1}^r}} ^{-1}} )  \nabla_\theta \ell(z_{M+2}^r,\hat{\theta}_{M}) \Big] \notag \\
& -  \frac{m^2}{(M+m)(M+2m)}  {\nabla}_\theta R(\hat{\theta}_{M})^\top ( {H_{\hat{\theta}_{M_m+{1}^g}} ^{-1}} - {H_{\hat{\theta}_{M_m+{1}^g}} ^{-1}} )\nabla_{\theta}^2\ell(z_{M+2}^r,\theta_M)^\top H_{\hat{\theta}_M}^{-1} \nabla_\theta \ell(z_{M+1}^g, \hat{\theta}_M)  \notag \\
& +  \frac{m^2}{(M+m)(M+2m)}  {\nabla}_\theta R(\hat{\theta}_{M})^\top  {H_{\hat{\theta}_{M_m+{1}^g}} ^{-1}} \nabla_{\theta}^2\ell(z_{M+2}^r,\theta_M)^\top H_{\hat{\theta}_M}^{-1}(\nabla_\theta  \ell(z_{M+1}^g, \hat{\theta}_M) -  \nabla_\theta  \ell(z_{M+1}^r, \hat{\theta}_M))+  \mathcal{O}\left( \frac{m^3}{(M+m)^3} \right) \label{eq:lem_m_p_eq10}
 \ealignt
 \normalsize
We do have:
\small
\balignt
{H_{\hat{\theta}_{M_m+{1}^g}} ^{-1}} - {H_{\hat{\theta}_{M_m+{1}^r}} ^{-1}}  = {H_{\hat{\theta}_{M_m+{1}^g}} ^{-1}}  (H_{\hat{\theta}_{M_m+{1}^r}} - H_{\hat{\theta}_{M_m+{1}^g}}) {H_{\hat{\theta}_{M_m+{1}^r}} ^{-1}}
\ealignt
\normalsize
If we assume that the Hessian inverse operator norm is bounded by $H$ and we have $\hat{\theta}_{M_m+{1}^g} - \hat{\theta}_{M_m+{1}^r} = \frac{m}{M+m} H_{\hat{\theta}_M}^{-1} \big(\nabla_{\theta} \ell (z^{r}_{M+1},\theta_M) - \nabla_{\theta} \ell (z^{g}_{M+1},\theta_M) \big)$. Now from the Lipschitz Hessian assumption:
\small
\balignt
\|{H_{\hat{\theta}_{M_m+{1}^g}} ^{-1}} - {H_{\hat{\theta}_{M_m+{1}^r}} ^{-1}}\| = {H_{\hat{\theta}_{M_m+{1}^g}} ^{-1}} ({H_{\hat{\theta}_{M_m+{1}^r}} }  - {H_{\hat{\theta}_{M_m+{1}^g}} } ) {H_{\hat{\theta}_{M_m+{1}^r}} ^{-1}}  \leq L H^2 \|\hat{\theta}_{M_m+{1}^g} - \hat{\theta}_{M_m+{1}^r}   \| \leq \frac{Gm}{M+m} \label{eq:lem_m_p_eq11}
\ealignt
\normalsize
for some constant $G$.
Now from the equation~\eqref{eq:lem_m_p_eq10} and \eqref{eq:lem_m_p_eq11}, we have the following:
\small
\balignt
R(\hat{\theta}_{M_m+{2}^g}) - \bbE[ R(\hat{\theta}_{M_m+{2}^r})] \leq \Delta_{\theta_{M_m+1^g}}' + \Delta_{\theta_{M}}'  + \frac{Gm^2}{(M+m)(M+2m)} + \mathcal{O}\left(\frac{m^3}{(M+m)(M+2m)^2}\right)
\ealignt
\normalsize
Last equation comes from equation~\eqref{eq:lem_m_p_eq11}  and from the fact that $\| \nabla_\theta  \ell(z_{M+1}^g, \hat{\theta}_M) - \nabla_\theta  \ell(z_{M+1}^r, \hat{\theta}_M) \| \in \mathcal{O}\left( \frac{m}{(M+2m)} \right)$. Hence for some constant $G'$:
\balignt
R(\hat{\theta}_{M_m+{2}^g}) - \bbE[ R(\hat{\theta}_{M_m+{2}^r})] \leq \Delta_{\theta_{M_m+1^g}}' + \Delta_{\theta_{M}}'  + \frac{G'm^2}{(M+m)(M+2m)}.
\ealignt
From lemma~\ref{lem:exact_vs_approx}, we know that $|R(\hat{\theta}_{M_m+{2}^g}) - R(\tilde{\theta}_{M_m+{2}^g}) | \in \mathcal{O}\left(  \frac{m^2}{(M+2m)^2}\right)$. Similarly $|R(\hat{\theta}_{M_m+{1}^g}) - R(\tilde{\theta}_{M_m+{1}^g}) | \in \mathcal{O}\left(  \frac{m^2}{(M+m)^2}\right)$.
Hence, we get our final result:
\balignt
R(\tilde{\theta}_{M_m+{2}^g}) - \bbE[ R(\hat{\theta}_{M_m+{2}^r})]\leq \Delta_{\theta_{M_m+1^g}}' + \Delta_{\theta_{M}}'  + \frac{G'm^2}{(M+m)(M+2m)} \leq  \Delta_{\theta_{M+1^g}}' + \Delta_{\theta_{M}}'  + \frac{G'm^2}{(M+m)^2}
\ealignt

\end{proof}

Next, we state the main result of the paper Theorem~\ref{thm:optimality_1_point_addition} about the average gain of greedy subset selection algorithm over random selection over the course of finite number of iterations.

\begin{theorem}[ Theorem~\ref{thm:optimality_m_point_addition}]
 If we run our greedy algorithm for $p$ successive steps and all the assumptions discussed in the section~\ref{subsec:assumptions} are satisfied then the gain in $ R $ after $p$ successive greedy steps  over the same number of random selection starting from the classifier $\theta_M$ can be characterized as follows:
 \balignt
 R(\tilde{\theta}_{M_m+{p}^g}) -  \bbE[ R(\hat{\theta}_{M_m+{p}^r})] \leq \sum_{i = 0}^{p-1}\Delta_{\theta_{M_m+i^g}}'  + \sum_{i = 1}^{p-1}\frac{\tilde{G}m^2}{(M+i*m)^2},
 \ealignt
 for some positive constant $\tilde{G}$ where $\tilde{\theta}_{M+{p}^g}$ is the output from the approximate algorithm~\ref{algo:approximate_greedy_m_1_rand} for $\epsilon = 0$.
 \end{theorem}

 \begin{proof}
Let us assume that $p$ is a multiple of $2$. The proof for more number of iterations can also be seen as just unrolling proof being done in the lemma~\ref{lem:m_1_2_cons_compare}.
\small
\balignt
\begin{split}\label{eq:thm_m_p_eq1}
R(\hat{\theta}_{M_m+{p}^g}) - \bbE[ R(\hat{\theta}_{M_m+{p}^r})] &=  \underbrace{R(\hat{\theta}_{M_m+{p}^g}) -  \bbE[R(\hat{\theta}_{M_m+{(p-1)}^g+1^r})]}_{:=\Delta_{\theta_{M_m+(p-1)^g}}'} +  \bbE[R(\hat{\theta}_{M_m+{(p-1)}^g+1^r})] -\bbE[ R(\hat{\theta}_{M_m+{p}^r})]  \\
&= \Delta_{\theta_{M_m+(p-1)^g}}' +   \bbE[R(\hat{\theta}_{M_m+{(p-1)}^g+1^r})] -\bbE[ R(\hat{\theta}_{M_m+{p}^r})]   \\
&= \Delta_{\theta_{M_m+(p-1)^g}}' +  \bbE[R(\hat{\theta}_{M_m+{(p-1)}^g+1^r})] -  R(\hat{\theta}_{M_m+{(p-1)}^g}) + R(\hat{\theta}_{M_m+{(p-1)}^g})  \\
& \qquad - \bbE[ R(\hat{\theta}_{M_m+{(p-1)}^r})] + \bbE[ R(\hat{\theta}_{M_m+{(p-1)}^r})]  -\bbE[ R(\hat{\theta}_{M_m+{p}^r})] \\
&= \Delta_{\theta_{M_m+(p-1)^g}}' +  \bbE[R(\hat{\theta}_{M_m+{(p-1)}^g+1^r})] -  R(\hat{\theta}_{M_m+{(p-1)}^g}) \\
&  - \left(\bbE[ R(\hat{\theta}_{M_m+{p}^r})] - \bbE[ R(\hat{\theta}_{M_m+{(p-1)}^r})]\right)+ \underbrace{R(\hat{\theta}_{M_m+{(p-1)}^g}) - \bbE[ R(\hat{\theta}_{M_m+{(p-1)}^r})]}_{:=\text{recursion term}}
\end{split}
\ealignt
\small

The bound on the term $\bbE[R(\hat{\theta}_{M_m+{(p-1)}^g+1^r})] -  R(\hat{\theta}_{M_m+{(p-1)}^g}) - \left(\bbE[ R(\hat{\theta}_{M_m+{p}^r})] - \bbE[ R(\hat{\theta}_{M_m+{(p-1)}^r})]\right)$ has been obtained in the proof of lemma~\ref{lem:m_1_2_cons_compare} which says:
\balignt
\bbE[R(\hat{\theta}_{M_m+{(p-1)}^g+1^r})] -  R(\hat{\theta}_{M_m+{(p-1)}^g}) - \left(\bbE[ R(\hat{\theta}_{M_m+{p}^r})] - \bbE[ R(\hat{\theta}_{M_m+{(p-1)}^r})]\right) \leq \frac{G'm^2}{(M+pm)(M+(p+1)m)}
\ealignt
for some real positive constant $G'$. Similarly, we can get the bound on the recursion term which essentially gives  the following result:
\balignt
R(\hat{\theta}_{M_m+{p}^g}) -  \bbE[ R(\hat{\theta}_{M_m+{p}^r})] \leq \sum_{i = 0}^{p-1}\Delta_{\theta_{M_m+i^g}}'  + \sum_{i = 1}^{p-1}\frac{{G}'m^2}{(M+i*m)^2}.
\ealignt
We finally apply lemma~\ref{lem:exact_vs_approx} to get our final result which is 
\balignt
R(\tilde{\theta}_{M_m+{p}^g}) -  \bbE[ R(\hat{\theta}_{M_m+{p}^r})] \leq \sum_{i = 0}^{p-1}\Delta_{\theta_{M_m+i^g}}'  + \sum_{i = 1}^{p-1}\frac{{G}m^2}{(M+i*m)^2}.
\ealignt
for some positive real constant $\tilde{G}$.

\end{proof}

\begin{proof}[Proof for Lemma~\ref{lem:compare_optimal}]
Given that $R(\theta_{M}) \leq \nu R^\star_M + \delta $.
\balignt
R(\theta_{M+1^g}) -R^\star_{M+1} &= R(\theta_{M+1^g})  - R(\theta_{M+1^w}) + R(\theta_{M+1^g}) - R^\star_{M+1} \\
&= R(\theta_{M+1^g})  - R(\theta_{M+1^w}) +  R(\theta_{M+1^w}) - R_M^\star + R_M^\star - R^\star_{M+1} \\
&= \Delta_{\theta_M}'' + R(\theta_{M+1^w}) -R(\theta_M) + R(\theta_M) - R_M^\star  + R_M^\star - R^\star_{M+1} \\ 
&\leq  \Delta_{\theta_M}'' + R(\theta_{M+1^w}) -R(\theta_M) + R(\theta_M) - R_M^\star  + R_{\hat{Z}_{M+1\backslash z^w}} - R^\star_{M+1} \\  
&\leq  \Delta_{\theta_M}'' + \nu R^\star_M + \delta +  R(\theta_{M+1^w}) -R(\theta_M) + R_{\hat{Z}_{M+1\backslash z^w}} - R^\star_{M+1} \\ 
&\leq  \Delta_{\theta_M}'' + \nu R^\star_M + \delta +  R(\theta_{M+1^w}) -R(\theta_M) - (R^\star_{M+1}  - R_{\hat{Z}_{M+1\backslash z^w}} ) \\ 
\ealignt
One can apply the first order approximation on the last two term to get the result.
\end{proof}

\subsection{Towards Model Selection from Multiple Models} \label{subsec:multiple_model}
As mentioned in the Section~\ref{sec:greedy_selection}, we can choose any objective $R(\hat \theta)$ to optimize. Consider a scenario where one has many models, which may be differentiable as well as non-differentiable. We would like to choose the model with maximal performance, however running all models on all the data is infeasible. Our method can be used to efficiently select points from the set based on some of the models which will generalize on other models  as discussed in section~\ref{subsec:model_transfer}.   Instead of risk, one can use the smooth version of median loss or ranking loss between the classifier's performance to select a point which preserves the ranking of the classifier or will select the point which will have the minimum median of the risk. The only criteria the new objective function which acts on the classifier's performance should fulfil is that it should be a differentiable objective with bounded Hessian and bounded gradient such that the first order approximation holds. However, the theoretical result in theorem~\ref{thm:optimality_1_point_addition} will still works with just the different characterization of the gain between greedy and random.

 \clearpage

\end{document}